\documentclass[11pt]{article}

\usepackage[preprint]{acl}

\usepackage{times}
\usepackage{latexsym}

\usepackage[T1]{fontenc}

\usepackage[utf8]{inputenc}

\usepackage{microtype}

\usepackage{inconsolata}

\usepackage{graphicx}

\usepackage{amsmath}
\usepackage{algorithm}
\usepackage{algpseudocode}
\usepackage{booktabs}
\usepackage{multirow}
\usepackage{makecell}
\usepackage{colortbl}
\usepackage{array}
\usepackage{subcaption}  
\usepackage{mwe}
\usepackage{amssymb}
\usepackage{amsthm}
\newtheorem{proposition}{Proposition}

\usepackage{xcolor}
\usepackage{pifont}
\usepackage{inconsolata}
\usepackage[T1]{fontenc}
\usepackage{tcolorbox}
\tcbuselibrary{skins}
\usepackage{fvextra}
\usepackage{alltt}

\definecolor{boxtitle}{HTML}{5C6B7A}   
\definecolor{boxback}{HTML}{EFF1F4}    
\definecolor{guardtext}{HTML}{0AA06E}  
\definecolor{passc}{HTML}{0AA06E}      
\definecolor{failc}{HTML}{C0392B}      
\newcommand{\gline}[1]{\textcolor{guardtext}{\bfseries #1}}

\newtcolorbox{exbox}[1]{
  enhanced,
  colback=boxback, colframe=boxtitle,
  coltitle=white, fonttitle=\bfseries\sffamily,
  colbacktitle=boxtitle,
  title={#1},
  boxrule=0.6pt, arc=2pt,
  left=6pt, right=6pt, top=5pt, bottom=5pt,
}

\title{Rethinking the Generation Order of Block Diffusion Language Models}

\author{
Kai Syun Hou \qquad James Kwok \\
Department of Computer Science and Engineering \\
The Hong Kong University of Science and Technology \\
}

\begin{document}
\maketitle
\begin{abstract}
Diffusion language models enable flexible arbitrary-order generation, 
but existing sampling methods are mostly designed for early masked diffusion models (MDMs).
In this work, we study sampling for recent block diffusion language models (BDLMs).
We show empirically and analytically that these models are naturally more aligned with left-to-right decoding than MDMs.
Based on this observation, we propose \underline{P}arallel \underline{A}uto\underline{r}egressive \underline{D}ecoding (PARD), a simple training-free sampling method that preserves left-to-right unmasking structure while allowing parallel token commitment.
Extensive experiments show that PARD consistently outperforms existing parallel samplers in generation quality, 
while achieving substantial speedups over pure AR decoding with only a small quality gap.
\end{abstract}

\section{Introduction}
\label{sec:intro}
Recently,
diffusion language models (DLMs) have 
emerged as a promising alternative to autoregressive language models (ARMs) for text generation \citep{dlmsurvey}.
In particular, masked diffusion models (MDMs)
have attracted growing attention for their simple denoising formulation and strong empirical performance~\citep{simple,shi2024simplified}.
MDMs iteratively reconstruct a sequence from a partially masked state, enabling tokens to be generated in arbitrary orders and in parallel.

The arbitrary-order generation flexibility of MDMs has motivated a growing line of sampling methods that seek to improve decoding quality and efficiency. 
These methods typically decide which masked positions should be unmasked at each denoising step \citep{dream,margin}.
More recent methods further accelerate decoding through parallel sampling, 
committing multiple positions per step based on criteria designed to balance generation speed and reliability ~\citep{fast-dllm,klass}.

Early MDMs, such as LLaDA~\citep{llada} and Dream~\citep{dream}, apply masked diffusion over the entire sequence. 
In contrast, 
Block diffusion language models (BDLMs) 
are trained to factorize generation autoregressively across blocks, 
while applying masked diffusion within each block conditioned on the clean preceding blocks \citep{blockdiffusion}.
This setting is increasingly important, as many recent high-performing DLMs adopt block diffusion, often by continuing training from pretrained autoregressive LLMs with a block-diffusion objective~\citep{fastdllmv2,sdar,llada20,llada21}. 
However, existing DLM sampling methods have mostly
been developed and evaluated on early MDMs, leaving their behavior on these recent BDLMs underexplored.

To investigate this gap, we conduct a comparison between two 8B models: SDAR \citep{sdar},
a recent BDLM, and LLaDA \citep{llada},
a representative MDM. 
We observe that both models exhibit an autoregressive-like decoding order.
Motivated by this, we compare a simple AR sampler, which always unmasks the leftmost masked token, against common arbitrary-order samplers. 
Surprisingly, this simple strategy achieves higher accuracy on SDAR, 
while it performs worse than the arbitrary-order samplers on LLaDA.

Besides the AR-pretrained initialization, we further explain this behavior by comparing the training-time input contexts of MDM and block diffusion.
We show that MDM training almost never exposes the model to the left-to-right context used by AR decoding.
In contrast, block diffusion training observes such contexts with larger and non-negligible probability,
and it biases the overall input context toward the left-to-right context.
This training-time bias makes BDLMs naturally more aligned with AR decoding.

Finally, we propose \emph{Parallel Autoregressive Decoding (PARD)}, a simple training-free sampling strategy that adapts existing arbitrary-order parallel samplers to left-to-right decoding. 
PARD unmasks the longest leftmost prefix of positions accepted by a given criterion, preserving AR structure while retaining parallelism.
Experiments on three recent BDLMs show that PARD consistently outperforms existing parallel samplers, 
achieving stronger generation quality while preserving substantial decoding speedups.

\section{Related Work}
\label{sec:related}

\subsection{Diffusion Language Models}
Diffusion models have emerged as successful generative models in continuous domains \citep{diffusion2015,ddpm,diffusion-survey}.
This success has motivated the development of discrete diffusion language models (DLMs) for text generation \citep{continuoustime,d3pm,lou2023discrete,ou2025your}.
Masked Diffusion Models (MDMs) are a particularly representative class among them \citep{simple,shi2024simplified}.
In the forward process of MDMs, clean tokens are gradually replaced with an absorbing state token \texttt{[MASK]}.
In the reverse process, the model learn to iteratively reconstruct the sequence by unmasking tokens until the original clean sequence is recovered.
Recent work has scaled MDMs to billions of parameters~\citep{llada,llada15},
showing that DLMs can achieve generation quality competitive with ARMs while offering the potential for faster inference~\citep{seeddiffusion,mercury}.

MDM treats the whole sequence as a single diffusion state,
so each denoising step processes all token positions.
Block diffusion \citep{blockdiffusion} instead factorizes generation autoregressively over consecutive blocks, while performing diffusion within each block.
Importantly, block diffusion language models (BDLMs) are not 
merely MDMs equipped with a semi-autoregressive (semi-AR) decoding strategy.
Rather, the model is trained with a masked diffusion objective within each block, conditioned on the clean prefix before the block.
This design enables exact KV cache reuse of finalized blocks, 
providing a substantial inference-speed advantage over MDMs,
which forward the full sequence at each step.
Moreover, BDLMs support arbitrary-length generation.
These practical advantages have motivated a wave of recent foundation BDLMs \citep{fastdllmv2, sdar, llada20, llada21, nbdiff, fastdvlm}.
Instead of training from scratch, 
these models are often adapted from strong ARMs through continued pretraining or fine-tuning with the block diffusion objective. 
Compared with MDMs of similar parameter scale, these BDLMs achieve substantially stronger generation quality and inference efficiency.

\begin{figure*}[t]
\centering
\begin{minipage}[t]{0.51\textwidth}
    \centering
    \includegraphics[width=\linewidth]{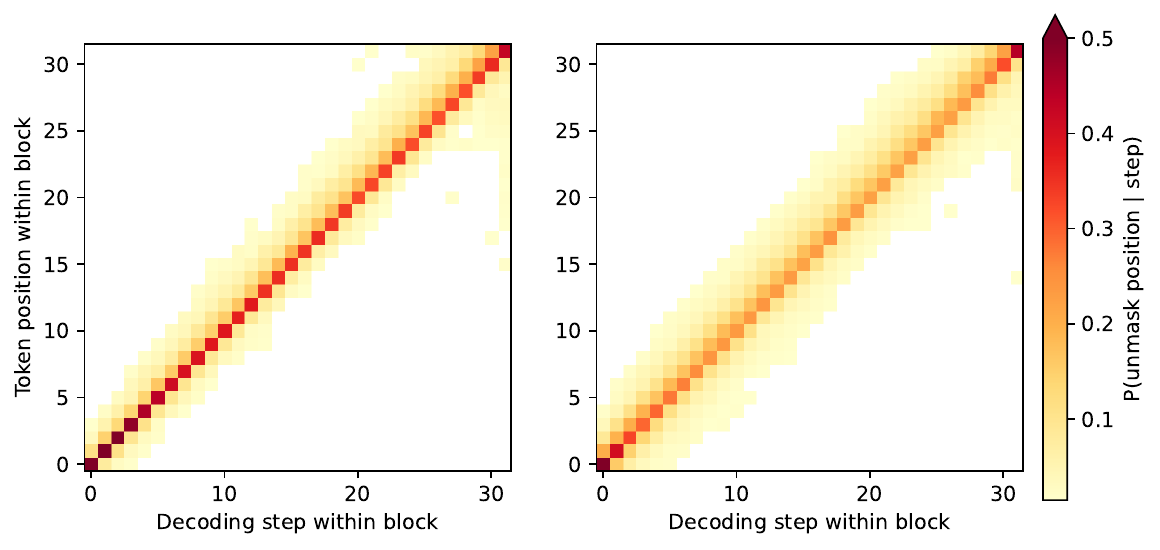}
\captionof{figure}{
Distributions
of unmasking positions at various denoising steps for SDAR (left) and LLaDA (right) on {\it HumanEval}.
    }
    \label{fig:heatmap-a}
\end{minipage}\hfill
\begin{minipage}[t]{0.47\textwidth}
    \centering
    \includegraphics[width=\linewidth]{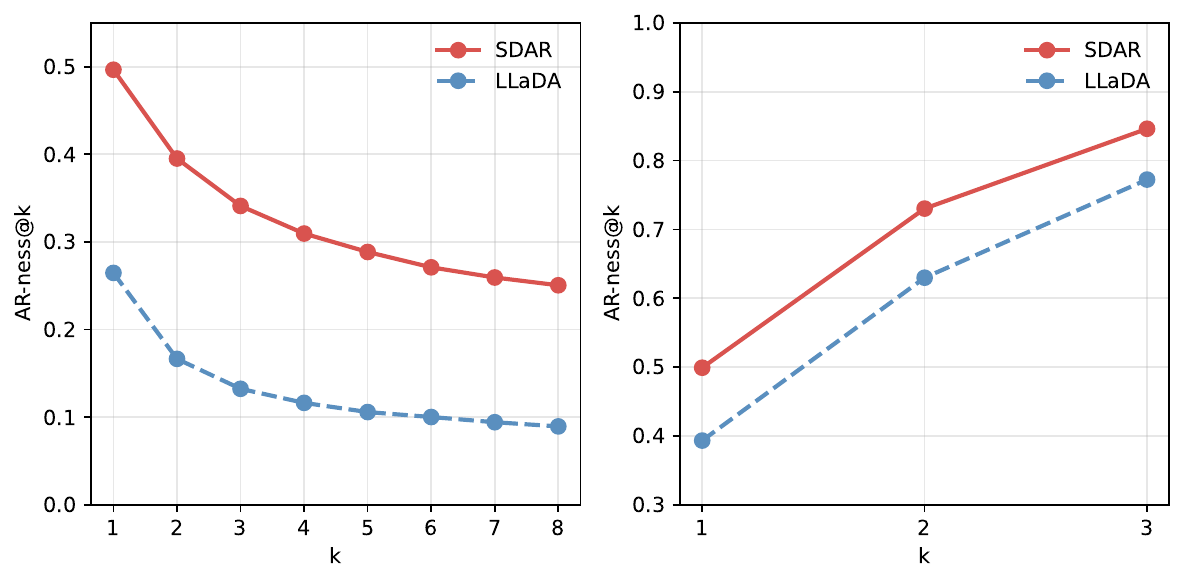}
    \captionof{figure}{
        Local (left) and global (right) AR-ness@$k$ of decoding on {\it HumanEval} under Confidence sampling.
    }
    \label{fig:ar-ness-curves}
\end{minipage}

\vspace{-12pt}
\end{figure*}

\subsection{Sampling Methods for MDMs}
Early MDMs use uniform sampling, which selects the next unmasked position uniformly from the remaining masked positions \citep{d3pm}. 
While aligned with the MDM training distribution, this strategy is often suboptimal for generation quality,
motivating heuristic samplers based on model-dependent scoring rules.
{\small\sf Confidence} selects the position with the highest predicted probability~\citep{maskgit}.
{\small\sf Margin} selects the position with the largest gap between the top-1 and top-2 predicted probabilities~\citep{margin}.
{\small\sf Entropy} selects the position with the lowest predictive entropy~\citep{dream}.
Uncode~\citep{uncode} augments confidence with a position-aware prior and an informativeness term that down-weights generic high-frequency tokens.

Recent works have also developed more advanced parallel sampling methods that
trade off generation quality and inference speed.
Representative approaches include confidence-thresholded unmasking \citep{fast-dllm, dimple}, KL-guided stability-based selection \citep{klass}, entropy-bounded subset selection \citep{ebsampler}, and tree-structured hierarchical decoding \citep{hierarchy}.

\section{Understanding the AR Bias of BDLMs}
\label{sec:method}

\subsection{Empirical Observations}
\label{sec:obs}
We first examine the decoding orders induced by the existing arbitrary-order sampler based on {\small\sf Confidence}.
We compare one recent BDLM, SDAR 8B \citep{sdar}, with the representative MDM, LLaDA 8B \citep{llada}.
Both models are evaluated on {\it HumanEval} using {\small\sf Confidence}.
For LLaDA, we employ the semi-AR decoding strategy.
We use $32$ denoising steps with block size $32$ for both models, so exactly one token is unmasked at each step.
For each generated block, we record the within-block position of the token unmasked at each step.

Figure~\ref{fig:heatmap-a}
shows the 
distribution of 
unmasking positions at each denoising step on {\it HumanEval}.
A strictly autoregressive order would place all probability mass on the diagonal.
As can be seen, both models exhibit a clear concentration around the diagonal, indicating that {\small\sf Confidence} already induces a near-left-to-right decoding order.
Within $\pm 2$ positions of the diagonal, the average probability mass is $0.75$ for SDAR and $0.70$ for LLaDA.
SDAR shows a stronger diagonal concentration, with an average on-diagonal mass of $0.39$ compared to $0.27$ for LLaDA (whose diagonal pattern is fainter).

To further quantify 
the decoding order's
closeness 
to AR, 
we compute the two AR-ness@$k$ metrics\footnote{Detailed definitions of the metrics are in Appendix \ref{sec:appendix_arness}.}
proposed in \cite{diffucoder}:
(i)
local AR-ness@$k$,
which measures the fraction of steps that extend a $k$-long consecutive run; and (ii)
global AR-ness@$k$,
which measures the fraction of steps where the selected position is among the $k$ earliest remaining masked positions.
Since blocks are generated autoregressively, 
we compute the 
average 
AR-ness metrics 
over all blocks.
As shown in Figure~\ref{fig:ar-ness-curves},
SDAR exhibits consistently higher local and global AR-ness 
across different $k$'s,
again confirming that 
BDLM is more AR-like than LLaDA
on decoding.

\begin{figure}[t]
\centering
\includegraphics[width=\linewidth]{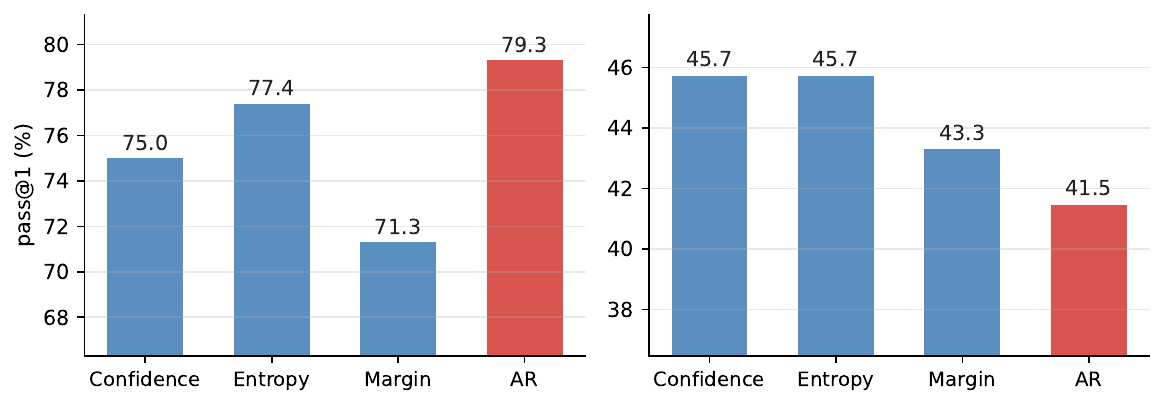}
\caption{{\it HumanEval} Base pass@1 for various samplers on SDAR (left) and LLaDA (right).}
\label{fig:ar-vs-conf-accuracy}
\vspace{-15pt}
\end{figure}

Recent works \citep{diffucoder,nap} have also observed that decoding in MDMs often exhibits strong
left-to-right structure despite their arbitrary-order flexibility. 
\citet{flexibilitytrap} further show that enforcing AR order during RL rollouts can improve the reasoning performance of MDMs.
However, it remains unclear whether explicit AR sampling at inference time is preferable,
especially for BDLMs.
We therefore compare three representative arbitrary-order samplers ({\small\sf Confidence}, {\small\sf Margin}, and {\small\sf Entropy}) with a pure AR sampler 
(which always selects the leftmost masked position at each step).
Figure~\ref{fig:ar-vs-conf-accuracy}
shows
the pass@1 values on {\it HumanEval} for SDAR and LLaDA with different samplers.
As can be seen, there is substantial difference between SDAR and LLaDA.
On SDAR, pure AR performs best, improving 
the strongest arbitrary-order sampler ({\small\sf Entropy}) by $1.9$ points.
In contrast, on LLaDA, pure AR performs worst,
underperforming {\small\sf Confidence} and {\small\sf Entropy} by $4.2$ points.
These results suggest that an explicit AR order is especially favorable for BDLMs, 
but not necessary for MDMs
trained from scratch.
In the following, 
we explore
why BDLMs are better aligned with AR sampling.

We additionally analyze the AR bias of Dream 7B \citep{dream} in Appendix \ref{sec:appendix_dream_ar}.

\subsection{Preliminaries}
\paragraph{Notation.}
Let $V$ be the vocabulary, $\texttt{[MASK]}$ be the absorbing mask token, and
$\bar{V} = V \cup \{\texttt{[MASK]}\}$.
A clean sequence $x_0 \in V^L$ is drawn from the data distribution $p_{\mathrm{data}}$.
The sequence is partitioned into $B$ contiguous blocks, each of size $L'$.
We denote each block as $x^b_0 \in V^{L'}$ ($b \in \{1,...,B\}$),
and write
$x_0^{<b}=(x_0^1,\ldots,x_0^{b-1})$ for the prefix before block $b$.
We denote $p_\theta(x_0^\ell \mid x_t)$ as the denoising model's predictive distribution at position $\ell$,
conditioned on the noised input sequence $x_t$.

\paragraph{MDM.}
During training, the forward process first samples a noise level $t\sim U[0,1]$,
then masks each position $\ell$ independently 
to obtain the noised sequence $x_t \in \bar{V}^{L}$ according to
\begin{equation}
\label{eq:forward}
q_t(x_t^\ell \mid x_0^\ell) =
\mathrm{Cat}\!\left(
x_t^\ell ;
\alpha_t \mathbf{e}_{x_0^\ell}
+
(1-\alpha_t)\mathbf{e}_{\texttt{[MASK]}}
\right),
\end{equation}
where $\alpha_t$ is a non-increasing noise schedule, and
$\mathbf{e}_v$ is the one-hot vector for token $v \in \bar{V}$.
Let $\mathcal{M}(x_t)=\{j:x_t^j=\texttt{[MASK]}\}$ be the masked positions in $x_t$.
With the linear schedule $\alpha_t=1-t$, 
MDM is trained to predict the original tokens at masked positions
\citep{simple}:
\[
\mathcal{L}_{\mathrm{MDM}}
=
-
\mathop{\mathbb{E}}\limits_{\substack{
t \sim U[0,1]\\
x_t \sim q_t(\cdot \mid x_0)
}}
\frac{1}{t}
\sum_{j \in \mathcal{M}(x_t)}
\log p_\theta(x_0^j \mid x_t).
\]
This is a NELBO upper bound on the negative log-likelihood
$-
\log p_\theta(x_0)$.

\paragraph{Block Diffusion.}
Block diffusion factorizes the sequence likelihood autoregressively over blocks:
$p_\theta(x_0) = \prod_{b=1}^B p_\theta(x_0^b \mid x_0^{<b})$ \citep{blockdiffusion}.
Each conditional distribution $p_\theta(x_0^b \mid x_0^{<b})$ is modeled by a masked diffusion process
within block $b$, conditioned on the clean prefix $x_0^{<b}$.
For each block $b$, the forward noising process first samples a noise level $t\sim U[0,1]$, then 
obtain $x_t^b$ 
by position-independent masking as \eqref{eq:forward}.
With the linear schedule $\alpha_t=1-t$,
the block diffusion training objective can be written as
\begin{eqnarray*}
\mathcal{L}_{\mathrm{Block}}
&= &
-
\sum_{b=1}^B
\mathbb{E}_{t
\sim U[0,1],
x_t^b
\sim q_t(\cdot \mid x_0^b)
}
\\
&&\Bigg[
\frac{1}{t}
\sum_{j \in \mathcal{M}(x_t^b)}
\log p_\theta(x_0^{b,j} \mid x_0^{<b}, x_t^b)
\Bigg],
\end{eqnarray*}
where $x_0^{b,j}$ is the $j$-th token in block $b$. 

\subsection{Why Are BDLMs More Aligned With AR?}
Recent BDLMs are obtained by fine-tuning or continual pretraining from AR-pretrained LLMs \citep{fastdllmv2,sdar}.
Compared with the AR-pretraining stage, which optimizes next-token prediction over trillions of tokens, block diffusion adaptation typically uses a much smaller token budget.
For example, Qwen2.5 7B is pretrained on 18T tokens~\citep{qwen2.5}, while Fast-dLLM v2 adapts it to a BDLM with only 1B tokens of fine-tuning.
This large training-budget gap suggests that block diffusion adaptation does not fully overwrite the left-to-right conditional behavior learned during AR pretraining.

Beyond initialization, the block diffusion objective also gives rise to input contexts better matched to AR decoding.
During block diffusion training, when the model predicts tokens in block $b$, it conditions on the clean prefix from all previous blocks and only 
the current block
is partially masked.
Future blocks are not included in the input.
In contrast, MDM training masks positions across the whole sequence, which can leave the model with a discontinuous left context while also exposing tokens to the right.

To quantify this difference,
consider a position $\ell$.
The mask pattern 
$\mathbf{m}^{\mathrm{AR}}_\ell \in \{0,1\}^N$
seen 
by an MDM 
(with $N=L$) or
block diffusion model
(with $N=bL'$ and position $\ell$
contained in block $b$) is purely AR when
\[
(\mathbf{m}^{\mathrm{AR}}_\ell)_j =
\begin{cases}
0, & j < \ell,\\
1, & \ell \le j \le N,
\end{cases}
\]
i.e., the left prefix has already been generated, 
while the future positions are not visible.
The following Proposition compares how often MDM and block diffusion
model see this AR mask pattern during training.
The proof is in Appendix \ref{sec:prop1_proof}.

\begin{proposition}[Probability of seeing AR mask pattern]
\label{prop:ar-mass}
Assume the linear noise schedule $\alpha_t = 1-t$ with $t\sim U[0,1]$.
Consider any position $\ell$.
We have
\[
\Pr \! \left(
\mathbf{m}^{\mathrm{AR}
}_\ell
\text{ observed} 
\right)
=
\frac{1}{(L+1)\binom{L}{\ell-1}}
\]
for MDM, and
\[
\Pr \! \left(
\mathbf{m}^{\mathrm{AR}}_{\ell} \text{ observed}\right) 
=
\frac{1}{(L'+1)\binom{L'}{k}},
\]
for block diffusion,
where 
$k=\ell-1-(b-1)L'$ is the 
left-context length
in this block.
\end{proposition}

For example, for MDM, when $L=1024$, the probability is approximately $10^{-310}$ at $\ell=512$, and 
approximately $10^{-6}$
even at the edge ($\ell=1024$).
Hence, AR contexts are almost absent during  MDM training.
In contrast, for block diffusion, 
with the same $\ell$, this probability is always larger.
For example, with $L'=32$,
it is about $10^{-11}$
at $k=15$;
about $0.03$ at 
$k=0$, and about $0.001$ at the second ($k=1$) and last positions ($k=31$).
Hence, during training, BDLMs receive more frequent supervision under contexts that match those queried by AR decoding.



Besides studying the probability that the AR mask pattern is exactly observed, we
can also measure how 
the 
mask
$\mathbf{M}$ 
(with $\mathbf{M}_j=1$ if position $j$ is masked, and $0$ otherwise)
used at position $\ell$ 
aligns with 
the AR pattern.
For MDM, 
$\mathbf{M}^{\text{MDM}}\in\{0,1\}^{L}$, and
the AR pattern alignment 
can be defined
as
\begin{align*}
\rho_\ell(\mathbf{M}^{\mathrm{MDM}})
&= \frac{1}{L-1}
\Bigg[
\sum_{j=1}^{\ell-1}(1-\mathbf{M}^{\mathrm{MDM}}_j) \\
&\qquad\qquad\quad
+ \sum_{j=\ell+1}^{L}\mathbf{M}^{\mathrm{MDM}}_j
\Bigg];
\end{align*}
whereas
for block diffusion
(with position $\ell$  contained in block $b$),
$\mathbf{M}^{\text{BDLM}}\in\{0,1\}^{bL'}$ and
the AR pattern alignment 
is
\begin{align*}
\rho_\ell(\mathbf{M}^{\mathrm{BDLM}})
&= \frac{1}{bL'-1}
\Bigg[
\sum_{j=1}^{\ell-1}(1-\mathbf{M}^{\mathrm{BDLM}}_j) \\
&\qquad\qquad\qquad
+ \sum_{j=\ell+1}^{bL'}\mathbf{M}^{\mathrm{BDLM}}_j
\Bigg].
\end{align*}
$\rho_\ell=1$ when the exact AR mask pattern is observed.
The following Proposition shows that 
block diffusion has a higher expected AR pattern alignment than MDM.
The proof is in Appendix~\ref{sec:prop2_proof}.

\begin{proposition}[Expected AR pattern alignment]
\label{prop:alignment}
Assume the linear noise schedule $\alpha_t=1-t$ with $t\sim U[0,1]$.
For MDM, the expected AR pattern alignment is
$\mathbb{E}
[\rho_{\ell}(\mathbf{M}^{\mathrm{MDM}})] = 1/2$.
For block diffusion, 
\begin{align*}
\mathbb{E}
[\rho_{\ell}(\mathbf{M}^{\mathrm{BDLM}})] &= \frac{1}{2} + \frac{(b-1)L'}{2(bL'-1)} \\
        & \ge
        \mathbb{E}
        [\rho_{\ell}(\mathbf{M}^{\mathrm{MDM}})],  
\end{align*}
where expectations are over the mask patterns.
\end{proposition}

The alignment increases with the block position $b$.
For the typical block size $L'{=}32$, the expected alignment is $0.5$ in
the first block, $0.75$ in the second block, and $0.95$ in the tenth block.

\section{Parallel Autoregressive Decoding}
\label{sec:ar-parallel}

The above analysis suggests that recent BDLMs are better matched to
left-to-right decoding contexts.
This motivates a sampler that preserves the AR prefix structure
during generation.
While a pure AR sampler obviously is the perfect match,
it is inherently
sequential and therefore sacrifices the parallelism advantage of DLMs.

To alleviate this problem,
we introduce Parallel Autoregressive Decoding (PARD).
PARD retains the AR bias while still allowing multiple tokens to be
unmasked per step, and
can be applied to existing arbitrary-order samplers.

Consider generation within block $b$, conditioned on the clean prefix $x_0^{<b}$.
At sampling step $s$, let $x_s^b$ be the current partially masked block, and
denote by $\mathcal{M}_s^b$ the set of masked positions in this block.
For each masked position $j \in \mathcal{M}_s^b$, let
\begin{equation}
    \pi_j(v)
    =
    p_\theta(x_0^{b,j}=v \mid x_0^{<b}, x_s^b).
\end{equation}
Let $p_j^{(1)}$ and $p_j^{(2)}$ be the largest and second-largest
probabilities under $\pi_j$, respectively.
Recall that the 
parallel samplers 
({\small\sf Confidence}, {\small\sf Margin}, and {\small\sf Entropy})
select 
the sets 
of 
unmasking 
positions as: 
\begin{eqnarray*}
\textsf{\small Confidence:}  &
    U_s^{\mathrm{C}} = \{j \in \mathcal{M}_s^b : c_j > \tau_c\}, \\
\textsf{\small Margin:} &
    U_s^{\mathrm{M}} = \{j \in \mathcal{M}_s^b : m_j > \tau_m\}, \\
\textsf{\small Entropy:} &
    U_s^{\mathrm{H}}
    = \{j \in \mathcal{M}_s^b : h_j < \tau_h\},
\end{eqnarray*}
where
$c_j = p_j^{(1)}$, $m_j = p_j^{(1)} - p_j^{(2)}$, 
$h_j = -\sum_{v \in V} \pi_j(v)\log \pi_j(v)$, and
$\tau_c, \tau_m, \tau_h$
are given thresholds.
If the selected set is empty, the sampler falls back to unmasking the top candidate
position under the same criterion, i.e., $\arg\max_j c_j$ for
{\small\sf Confidence}, $\arg\max_j m_j$ for {\small\sf Margin}, and $\arg\min_j h_j$ for {\small\sf Entropy}.

PARD applies the same criteria, but restricts unmasking to the
leftmost accepted prefix.
It scans the masked positions from left to
right and unmasks consecutive positions until the first one that fails the
criterion.
Specifically, it selects the sets of unmasking positions as:
\begin{align*}
    U_s^{\mathrm{AR}\text{-}\mathrm{C}}
    &=
    \{j \in \mathcal{M}_s^b :
    c_i > \tau_c \ \forall i \in \mathcal{M}_s^b,\, i \leq j\}, \\
    U_s^{\mathrm{AR}\text{-}\mathrm{M}}
    &=
    \{j \in \mathcal{M}_s^b :
    m_i > \tau_m \ \forall i \in \mathcal{M}_s^b,\, i \leq j\}, \\
    U_s^{\mathrm{AR}\text{-}\mathrm{H}}
    &=
    \{j \in \mathcal{M}_s^b :
    h_i < \tau_h \ \forall i \in \mathcal{M}_s^b,\, i \leq j\}.
\end{align*}
If this prefix is empty, it falls back to unmasking the
leftmost masked position, i.e., $U_s=\{\min \mathcal{M}_s^b\}$.

Let $\hat{x}_s^{b,j} = \arg\max_{v \in V} p_\theta(x_0^{b,j}=v \mid x_0^{<b}, x_s^b)$.
The block is then updated as
\[
x_{s+1}^{b,j} =
\begin{cases}
\hat{x}_s^{b,j} & j \in U_s, \\
x_s^{b,j} & \text{otherwise.}
\end{cases}
\]
where $U_s$ is chosen from the corresponding arbitrary-order or PARD's unmasking sets defined above.
The update is repeated until all positions in the current block are unmasked.
Generation then proceeds block by block until an \texttt{[EOS]} token is generated or the maximum generation length is reached.

\section{Experiments}
\label{sec:experiments}

\begin{figure*}[t]
    \centering
    \includegraphics[width=\linewidth]{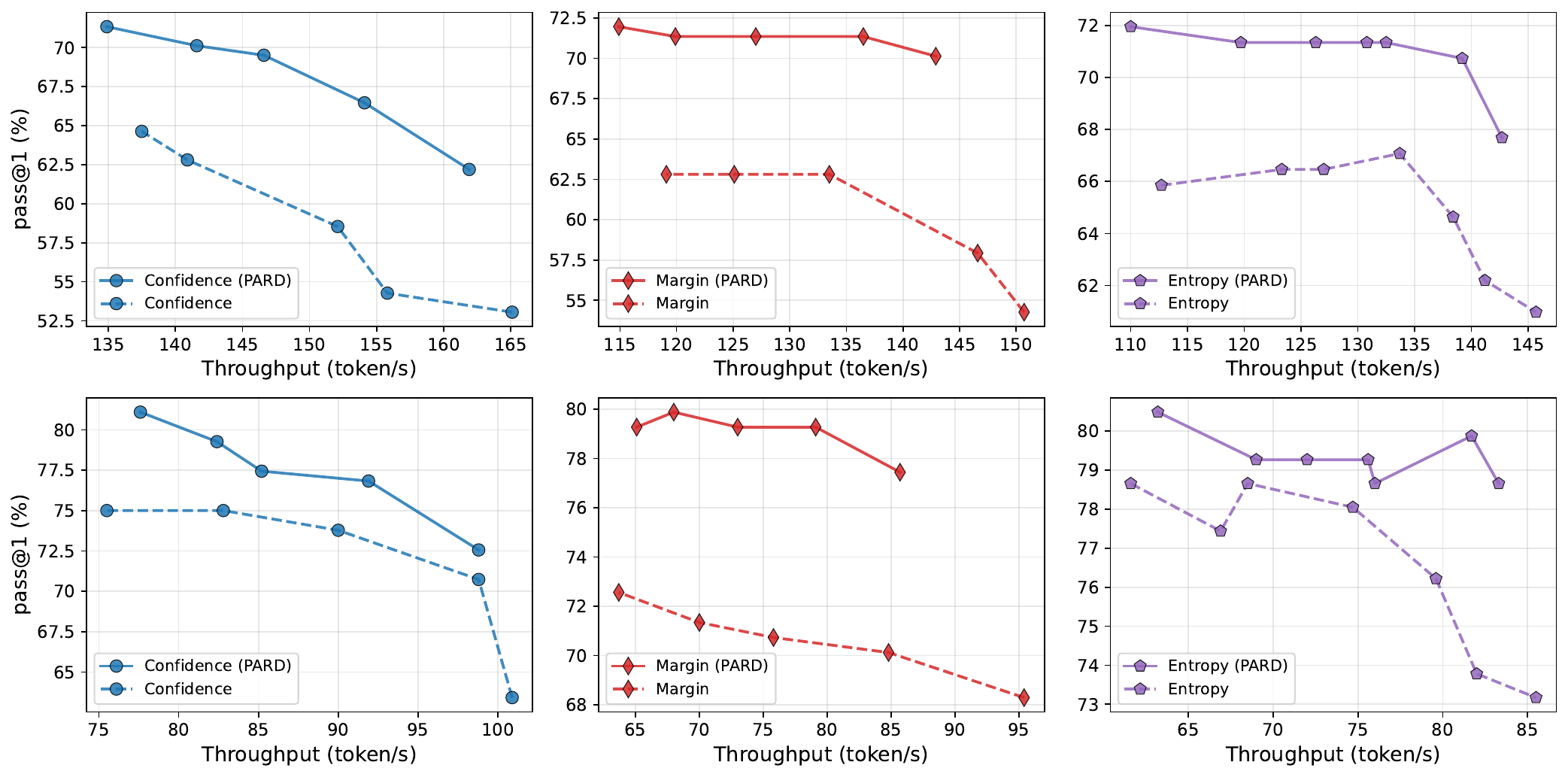}
    \captionof{figure}{
        {\it HumanEval} Base pass@1 vs.\ decoding throughput (token/s) for
        {\small\sf Confidence} (left), {\small\sf Margin} (middle), and {\small\sf Entropy} (right) parallel sampling and their PARD
        variants on Fast-dLLM v2 7B (top row) and SDAR 8B (bottom row).
    }
    \label{fig:humaneval-tps-2x3}
    \vspace{-8pt}
\end{figure*}

\begin{figure}[t]
    \centering
    \includegraphics[width=\linewidth]{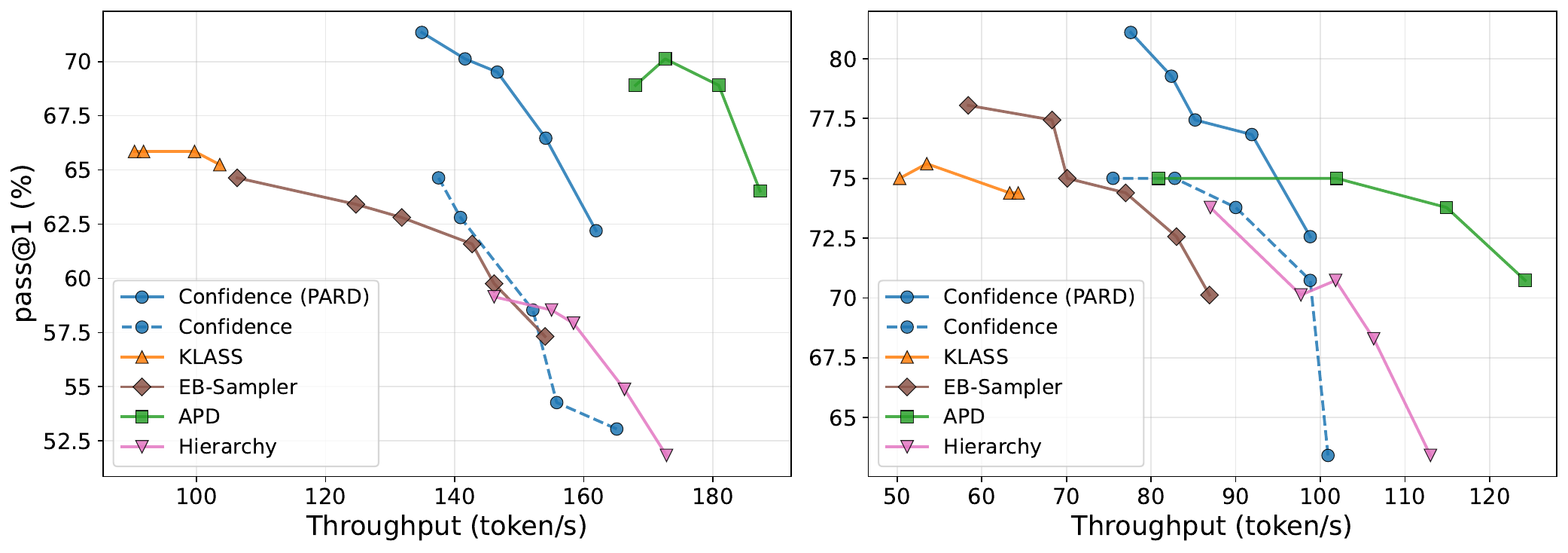}
    \captionof{figure}{
        {\it HumanEval} Base pass@1 vs.\ decoding throughput (token/s)
        for confidence-based PARD, EB-Sampler, KLASS, APD, and Hierarchy
        on Fast-dLLM v2 7B (left) and SDAR 8B (right).
    }
    \label{fig:humaneval-tps-1x2}
    \vspace{-8pt}
\end{figure}

In this section,
we 
first
evaluate the speed-quality trade-off by varying the parallel decoding thresholds.
Next, we focus on
the generation qualities 
and
evaluate 
across various tasks and models.

We evaluate three recent BDLMs: Fast-dLLM v2 7B~\citep{fastdllmv2},
SDAR 8B~\citep{sdar}, and LLaDA2.1-Mini 16B~\citep{llada21}.
All models use a block size of $32$, matching their default configurations.
For Fast-dLLM v2, we follow its default sub-block inference strategy with sub-block size $8$.
LLaDA2.1-Mini additionally supports token editing,
where previously unmasked tokens can be revised when a new prediction exceeds an editing threshold.
More details about the models are in Appendix \ref{sec:appendix_checkpoints}.

\subsection{Speed-Quality Trade-Off}
\label{sec:tradeoff}

We first evaluate PARD on
how it improves the trade-off between generation quality and decoding efficiency.
We compare {\small\sf Confidence}, {\small\sf Margin}, and {\small\sf Entropy} parallel sampling with their PARD variants,
sweeping the corresponding thresholds $\tau_c$, $\tau_m$, and $\tau_h$.
Details on the threshold values are in Appendix~\ref{sec:appendix_thresholds}.
Evaluation is based on
Fast-dLLM v2 and SDAR, 
conducted on the {\it HumanEval} Base with a batch size of 4.

Figure~\ref{fig:humaneval-tps-2x3} shows the pass@1-throughput curves. Results on 
the pass@1-Number of Function Evaluations (NFE) curves are shown in 
Figure \ref{fig:humaneval-nfe-2x3} 
of Appendix~\ref{sec:appendix_tradeoff}.
As can be seen, applying the AR prefix constraint consistently improves the trade-off,
shifting the curves toward higher pass@1 at comparable throughput.
{\small\sf Margin} shows the largest gap between the original sampler and its PARD variant.
This suggests that a token-level margin criterion alone is insufficient for parallel sampling,
but becomes much more effective when paired with a left-to-right prefix constraint.
The gains are also clear for {\small\sf Confidence} and {\small\sf Entropy}, indicating that 
PARD is 
advantageous 
for various criteria.

\begin{table*}[t]
\centering
\small
\resizebox{\textwidth}{!}{%
\begin{tabular}{llccccccccc}
\toprule
\multirow{2}{*}{\textbf{Type}} & \multirow{2}{*}{\textbf{Method}} & \multicolumn{2}{c}{\textbf{HumanEval}} & \multicolumn{2}{c}{\textbf{MBPP}} & \multirow{2}{*}{\textbf{GSM8K}} & \multirow{2}{*}{\textbf{MATH}} & \multirow{2}{*}{\textbf{IFEval}} & \multirow{2}{*}{\textbf{BBH}} & \multirow{2}{*}{\textbf{Avg}} \\
\cmidrule(lr){3-4} \cmidrule(lr){5-6}
 & & Base & Plus & Base & Plus & & & & & \\
\midrule
\rowcolor{gray!15}
\multicolumn{11}{c}{\textit{\textbf{Fast-dLLM v2 7B}}} \\
\multirow{5}{*}{Sequential}
& Confidence           & 65.9 & 61.0 & 63.5 & 53.4 & 83.6 & 60.7 & 60.1 & 51.9 & 62.5 \\
& Margin               & 62.8 & 57.9 & 60.3 & 50.8 & 82.9 & 60.0 & 57.9 & 50.7 & 60.4 \\
& Entropy              & 65.9 & 61.0 & 65.3 & 56.1 & 82.6 & 62.4 & 60.8 & 53.0 & 63.4 \\
& Uncode               & 68.9 & 64.6 & 67.2 & 56.4 & 84.1 & 63.4 & 62.3 & 54.0 & 65.1 \\
& AR                   & \underline{72.0} & \underline{67.7} & \underline{69.3} & \underline{58.7} & \underline{84.3} & \underline{65.0} & \underline{62.5} & \underline{54.6} & \underline{66.8} \\
\hline
\multirow{4}{*}{Parallel}
& KLASS                & 65.9 & 61.0 & 63.5 & 53.4 & \textbf{83.8} & 61.0 & 60.3 & 52.1 & 62.6 \\
& EB-Sampler           & 62.8 & 59.1 & 63.8 & 54.5 & 82.3 & 58.7 & 61.0 & 51.7 & 61.7 \\
& Confidence Parallel  & 64.6 & 59.8 & 63.5 & 52.9 & 82.5 & 58.0 & \textbf{62.1} & 51.4 & 61.9 \\
& APD          & 64.0 & 59.8 & 64.0 & 54.8 & 81.3 & 58.8 & 58.4 & 49.8 & 61.3 \\
& Hierarchy         & 59.1 & 54.9 & 61.1 & 51.3 & 81.5 & 58.1 & 59.7 & 50.2 & 59.5 \\
& PARD    & \textbf{71.3} & \textbf{67.1} & \textbf{69.6} & \textbf{58.7} & \textbf{83.8} & \textbf{63.5} & \textbf{62.1} & \textbf{54.3} & \textbf{66.3} \\

\midrule
\rowcolor{gray!15}
\multicolumn{11}{c}{\textit{\textbf{SDAR 8B}}} \\
\multirow{5}{*}{Sequential}
& Confidence           & 75.0 & 70.1 & 75.1 & 64.0 & 89.8 & 71.4 & 58.8 & 62.2 & 70.8 \\
& Margin               & 71.3 & 65.2 & 73.3 & 61.4 & 89.5 & 70.3 & 56.9 & 61.8 & 68.7 \\
& Entropy              & 77.4 & 71.3 & 77.0 & 65.6 & 91.0 & 72.1 & 59.7 & 65.5 & 72.5 \\
& Uncode               & 78.0 & 72.6 & 76.2 & 65.1 & \underline{91.2} & \underline{73.3} & 60.1 & 66.2 & 72.8 \\
& AR                   & \underline{79.3} & \underline{73.2} & \underline{78.0} & \underline{66.4} & 91.1 & \underline{73.3} & \underline{61.0} & \underline{66.5} & \underline{73.6} \\
\hline
\multirow{4}{*}{Parallel}
& KLASS                & 75.0 & 70.1 & 75.1 & 64.3 & 89.9 & 70.2 & 57.9 & 61.5 & 70.5 \\
& EB-Sampler           & 75.0 & 69.5 & \textbf{78.0} & \textbf{66.9} & 90.0 & 69.6 & 59.5 & 64.7 & 71.7 \\
& Confidence Parallel  & 75.0 & 70.1 & 74.9 & 63.8 & 89.1 & 67.7 & 58.8 & 61.9 & 70.2 \\
& APD          & 70.7 & 66.5 & 72.5 & 61.4 & 89.0 & 67.8 & 59.3 & 61.1 & 68.5 \\
& Hierarchy         & 73.8 & 68.3 & 73.5 & 62.7 & 87.7 & 65.8 & 58.2 & 61.6 & 69.0 \\
& PARD    & \textbf{81.1} & \textbf{76.2} & 77.0 & 65.3 & \textbf{91.4} & \textbf{71.6} & \textbf{60.4} & \textbf{65.8} & \textbf{73.6} \\

\midrule
\rowcolor{gray!15}
\multicolumn{11}{c}{\textit{\textbf{LLaDA2.1-Mini 16B}}} \\
\multirow{5}{*}{Sequential}
& Confidence           & 75.0 & 72.6 & 82.3 & 69.8 & 91.7 & 81.8 & 83.9 & 86.3 & 80.4 \\
& Margin               & 78.1 & 74.4 & 81.5 & 68.8 & 91.4 & 82.1 & 83.7 & 86.8 & 80.8 \\
& Entropy              & 79.3 & 75.6 & 82.8 & 69.3 & 91.4 & 81.9 & 84.3 & \underline{87.1} & 81.5 \\
& Uncode               & 76.2 & 72.0 & 83.6 & 71.2 & 92.0 & \underline{82.3} & \underline{84.7} & 86.7 & 81.1 \\
& AR                   & \underline{87.8} & \underline{84.8} & \underline{85.7} & \underline{72.2} & \underline{92.3} & 82.2 & 84.3 & 86.9 & \underline{84.5} \\
\hline
\multirow{4}{*}{Parallel}
& KLASS                & 78.7 & 76.2 & 80.7 & 68.2 & 91.3 & \textbf{82.3} & \textbf{84.1} & 85.8 & 80.9 \\
& EB-Sampler           & 80.5 & 77.4 & 81.5 & 69.6 & 91.4 & \textbf{82.3} & 83.0 & 86.0 & 81.5 \\
& Confidence Parallel  & 84.8 & 82.3 & 78.8 & 66.7 & 90.5 & 81.6 & 81.9 & 85.2 & 81.5 \\
& Hierarchy         & 81.1 & 76.8 & 79.4 & 67.2 & 91.6 & 81.8 & 80.4 & 85.6 & 80.5 \\
& PARD    & \textbf{86.6} & \textbf{82.9} & \textbf{84.9} & \textbf{71.4} & \textbf{92.0} & \textbf{82.3} & \textbf{84.1} & \textbf{86.4} & \textbf{83.9} \\
\bottomrule
\end{tabular}
}
\caption{Performance across six benchmarks on three recent BDLMs.
Methods are grouped into sequential and parallel samplers.
For each benchmark, the best result among parallel samplers is \textbf{bolded}, and the best result among sequential samplers is \underline{underlined}.}
\label{tab:main_results}
\vspace{-1pt}
\end{table*}

We further compare confidence-based PARD with advanced parallel samplers, KLASS \cite{klass}, EB-Sampler \cite{ebsampler}, APD \cite{apd}, and Hierarchy decoding \cite{hierarchy}.
APD is particularly related to PARD because it also unmasks a longest
agreed prefix, where the prefix is verified by an external small AR model.
We vary their hyperparameters, 
details are provided in Appendix~\ref{sec:appendix_thresholds}.


Figure~\ref{fig:humaneval-tps-1x2} shows the pass@1-throughput curves.
Results on the pass@1-NFE curves are in Figure~\ref{fig:humaneval-nfe-1x2} of Appendix~\ref{sec:appendix_tradeoff}.
Across both models, PARD better preserves pass@1 while improving throughput. 
KLASS has the lowest throughput, and its throughput changes only slightly
across KL thresholds. 
This is because its KL-stability criterion requires
high-precision full-vocabulary softmax operations for numerical stability.
APD achieves higher throughput than PARD on Fast-dLLM v2,
but its best pass@1 remains lower than that of PARD on both models.
Overall, these results suggest that preserving the AR prefix structure is
more efficient and effective for BDLM decoding than using more complex arbitrary-order subset-selection rules.
We observe similar trends on \textit{MBPP} Base, with results
reported in Appendix~\ref{sec:appendix_tradeoff}.

\begin{table*}[t]
\centering
\small
\resizebox{\textwidth}{!}{%
\begin{tabular}{llcccccc}
\toprule
\textbf{Model} & \textbf{Method} & \textbf{HumanEval} & \textbf{MBPP} & \textbf{GSM8K} & \textbf{MATH} & \textbf{IFEval} & \textbf{BBH} \\
\midrule
\multirow{2}{*}{Fast-dLLM v2 7B}
  & AR                &  85.8 &  82.7 &  96.1 &  69.6 &  64.8 &  77.6 \\
  & PARD & 134.9 (1.57$\times$) & 112.5 (1.36$\times$) & 179.9 (1.87$\times$) & 152.0 (2.18$\times$) & 111.2 (1.72$\times$) & 113.5 (1.46$\times$) \\
\midrule
\multirow{2}{*}{SDAR 8B}
  & AR                &  46.5 &  42.4 &  46.2 &  40.7 &  32.0 &  40.1 \\
  & PARD &  77.6 (1.67$\times$) &  72.6 (1.71$\times$) & 108.2 (2.34$\times$) & 124.9 (3.07$\times$) &  59.4 (1.86$\times$) &  60.5 (1.51$\times$) \\
\midrule
\multirow{2}{*}{LLaDA2.1-Mini}
  & AR                &  12.1 &  10.6 &  10.0 &  10.9 &   9.1 &  11.6 \\
  & PARD &  44.1 (3.64$\times$) &  36.8 (3.47$\times$) &  29.4 (2.94$\times$) &  28.9 (2.65$\times$) &  11.7 (1.29$\times$) &  25.2 (2.17$\times$) \\
\bottomrule
\end{tabular}
}
\caption{Decoding throughput (tokens/sec) of AR and confidence-based PARD on six benchmarks.
Speedup over AR is shown in parentheses.}
\label{tab:speed}
\vspace{-1pt}
\end{table*}

\subsection{Generation Quality and Efficiency}
\label{sec:benchmark_performance}

Instead of sweeping the thresholds $\tau_c,
\tau_m,
\tau_h$ as in 
Section~\ref{sec:tradeoff},
in this section we focus on the commonly-used threshold settings and
evaluate generation quality across a broader set of tasks and models.

We use six benchmarks spanning code
generation, mathematical reasoning, instruction following, and general
reasoning.
For code generation, we evaluate {\it HumanEval}~\citep{humaneval} and
{\it MBPP}~\citep{mbpp}, reporting pass@1 on both the Base and Plus variants.
For mathematical reasoning,  we use
{\it GSM8K}~\citep{gsm8k} and {\it MATH}~\citep{math};
for instruction following,
{\it IFEval}~\citep{ifeval};
and for general reasoning,
{\it BigBenchHard (BBH)}~\citep{bbh}.
For LLaDA2.1-Mini, 
we evaluate \textit{BBH} on an 8-task subset
due to computational constraints.
Details are in Appendix \ref{sec:appendix_benchmarks}.

For models,
in addition to Fast-dLLM v2 7B and SDAR 8B, we include
LLaDA2.1-Mini 16B~\citep{llada21}, a recent BDLM which supports token editing.
Since these BDLMs use {\small\sf Confidence Parallel}~\citep{fast-dllm}
as their default decoding strategy,
we compare against it and use confidence-based PARD in this benchmark evaluation.
Following their default configurations, we
use $\tau_c{=}0.9$ for Fast-dLLM v2 and SDAR, and $\tau_c{=}0.7$ for LLaDA2.1-Mini.
For LLaDA2.1-Mini, we use the default editing threshold
$\tau_{\mathrm{edit}}{=}0.5$.

We include the sequential strategies of {\small\sf Confidence}, {\small\sf Margin} and {\small\sf Entropy} as additional baselines, each of which unmasks one token per step.
We also compare with pure AR sampling, which always unmasks the leftmost masked
position, and Uncode~\citep{uncode}, a recent strong sequential sampler with 
position-aware and token informativeness prior.
For the parallel baselines KLASS, EB-Sampler, and APD, 
we follow their default settings, 
using a KL-divergence threshold of $0.01$, an entropy-bound threshold of $0.1$, and a mixture weight of 0.5, respectively. 
We report APD only for Fast-dLLM v2 and SDAR, where a small AR verifier sharing the same tokenizer as the target DLM is available.
For Hierarchy decoding, we use the default low confidence threshold of $0.5$ \cite{hierarchy},
while setting the high confidence threshold to the same value as $\tau_c$, 
as they serve the same role.

Table~\ref{tab:main_results}
shows the results of various sampling methods.
Among the parallel samplers, PARD
achieves the best average accuracy on every model, outperforming {\small\sf Confidence Parallel} by $4.4$, $3.4$, and $2.4$ points on Fast-dLLM v2, SDAR, and LLaDA2.1-Mini, respectively.
The improvements are more pronounced on code generation tasks.
This may be because arbitrary-order samplers tend to bypass uncertain decision points and commit easier later tokens first~\citep{flexibilitytrap}, which can prematurely constrain the program structure.
We provide a qualitative example in Appendix~\ref{sec:appendix_case_study}
illustrating this behavior.
Taking the sequential samplers also into account,
pure AR achieves the best accuracy on all three models.
PARD performs strongly, matching AR 
on SDAR and trailing AR by only $0.5$ and $0.6$ points on
Fast-dLLM v2 and LLaDA2.1-Mini, respectively. 
We additionally evaluate with stochastic top-$p$ sampling in Appendix \ref{sec:appendix_top-p},
where PARD retains the highest mean performance among the competitive parallel samplers.

Although AR achieves the strongest generation quality, it is sequential.
We measure the decoding throughput of AR and confidence-based PARD across
benchmarks using a batch size of $4$.
Table~\ref{tab:speed} reports throughput in generated tokens per second, with
speedup over AR shown in parentheses.
Confidence-based PARD is consistently faster than AR across the
benchmarks, often by a large margin, 
achieving up to $2.18\times$, $3.07\times$, and $3.64\times$ speedup on Fast-dLLM v2, SDAR, and LLaDA2.1-Mini, respectively.
These results highlight the role of PARD as an efficiency-oriented alternative to pure AR sampling.
We further show the distribution of the number of tokens unmasked per step 
of PARD in
Appendix~\ref{sec:appendix_num_unmasked}.

\begin{table}[t]
\centering
\small
\resizebox{\linewidth}{!}{%
\begin{tabular}{lccc}
\toprule
\textbf{Method} & \textbf{Fast-dLLM v2 7B} & \textbf{SDAR 8B} & \textbf{LLaDA2.1-Mini} \\
\midrule
PARD         & 69.6 & 77.0 & 84.9 \\
\midrule
Leftmost ($k{=}2$)       & 49.5 & 57.7 & 76.5 \\
Leftmost ($k{=}3$)       & 34.4 & 34.7 & 69.0 \\
Leftmost ($k{=}4$)       & 22.8 & 23.5 & 61.6 \\
\bottomrule
\end{tabular}
}
\caption{\textit{MBPP} Base pass@1 for PARD and static leftmost variants on Fast-dLLM v2 7B, SDAR 8B, and LLaDA2.1-Mini.
}
\label{tab:static_parallel}
\vspace{-5pt}
\end{table}

\subsection{Ablation }
In this experiment, we ablate the effect of adaptive prefix-length selection by
comparing confidence-based PARD with a static leftmost variant, 
which always unmasks a fixed number ($k$) of leftmost masked
tokens per step.
Evaluation is performed
on {\it MBPP}.
PARD uses the same thresholds as in Section \ref{sec:benchmark_performance}.

As shown in Table~\ref{tab:static_parallel}, the static variants suffer large performance drops across all three models.
Even with only $k=2$, pass@1 drops by 20.1 points on Fast-dLLM v2 7B, 19.3 points on SDAR 8B, and 8.4 points on LLaDA2.1-Mini.
The degradation becomes even more severe as $k$ increases, with $k=4$ reducing pass@1 by 46.8, 53.5, and 23.3 points, respectively.
These results show that simply committing multiple leftmost tokens is too aggressive.
Effective parallel left-to-right decoding must also decide when the leftmost prefix is sufficiently reliable.


\section{Conclusion}
\label{sec:conclusion}
We revisited sampling for recent BDLMs in this work.
Both empirical observations and training-context analysis suggest that BDLMs are better matched to left-to-right decoding than to fully arbitrary-order sampling.
To recover parallel decoding efficiency, we introduced PARD, which preserves the AR bias while committing multiple confident tokens per step.
Across three recent BDLMs and six benchmarks, PARD substantially improves throughput over pure AR while incurring only a small quality drop,
and consistently outperforms existing parallel samplers in generation quality.

\section*{Limitations}
Our study focuses on inference-time sampling for recent block diffusion language models, and does not modify the training objective. 
PARD is designed as a simple plug-and-play sampler for existing BDLMs, so we do not consider training-based approaches such as sampling-aware fine-tuning.
These directions are complementary to our work and may further improve the quality--efficiency trade-off. 
In addition, due to computational constraints, our experiments focus on publicly available BDLMs in the 7B--16B scale, and we do not evaluate larger models such as 100B-scale BDLMs.
Studying whether the same decoding behavior and speed--quality trade-off persist at larger scales is an interesting direction for future work.


\bibliography{references}


\appendix
\label{sec:appendix}

\section{Definitions of Local AR-ness@$k$ and Global AR-ness@$k$}
\label{sec:appendix_arness}
Let $L'$ be the block size.
For a generated block, let $M_{s-1} \subseteq \{1,\ldots,L'\}$ denote the set of masked positions before step $s$,
and let $p_s \in M_{s-1}$ be the position unmasked at step $s$.
Following \citet{diffucoder}, we define local AR-ness by the step-level indicator
\[
I_{\mathrm{loc}}(s,k)=
\begin{cases}
1, & \{p_{s-i}\}_{i=1}^{k}=\{p_s-i\}_{i=1}^{k},\\
0, & \text{otherwise}.
\end{cases}
\]
The block-level local AR-ness@$k$ is computed as
\[
\frac{1}{L'} \sum_{s=1}^{L'} I_{\mathrm{loc}}(s,k).
\]

For global AR-ness, let $\mathrm{Left}_k(M_{s-1})$ denote the $k$ leftmost positions in $M_{s-1}$.
We define
\[
I_{\mathrm{glob}}(s,k)=
\begin{cases}
1, & p_s \in \mathrm{Left}_k(M_{s-1}),\\
0, & \text{otherwise}.
\end{cases}
\]
The block-level global AR-ness@$k$ is computed as
\[
\frac{1}{L'}\sum_{s=1}^{L'} I_{\mathrm{glob}}(s,k).
\]

\section{Proofs}

\subsection{Proof of Proposition~\ref{prop:ar-mass}}
\label{sec:prop1_proof}
\begin{proof}
Given a clean training sequence $x_0$.
Assume $\alpha_t = 1-t$.
During training, we draw $t \sim U[0,1]$, then draw the noised sequence $x_t \sim q(\cdot \mid x_0)$.
Each token position is then masked independently by probability $1-\alpha_t = t$.
Let $E = \{\mathbf{m}_\ell^{\mathrm{AR}}\ observed\}$ denote the event of interest.

For the MDM case, we have
\begin{align}
    \Pr(E \mid t) &= \alpha_t^{\ell-1} \cdot (1-\alpha_t)^{L-\ell+1} \\
    &= (1-t)^{\ell-1} \cdot t^{L-\ell+1}
\end{align}
Then marginalize over $t \sim U[0,1]$,
\begin{align}
    \Pr(E) &= \mathbb{E}_{t\sim U[0,1]} \Pr(E \mid t) \\
    &= \int_0^1 (1-t)^{\ell-1} \cdot t^{L-\ell+1} dt \\
    &= \frac{(\ell-1)! \, (L-\ell+1)!}{(L+1)!} \\ \label{beta}
    &= \frac{1}{(L+1) \binom{L}{\ell-1}}
\end{align}
where \eqref{beta} uses the definition of the Beta function.

For the block diffusion case, 
let $k=\ell-1-(b-1)L'$ be the number of positions before $\ell$ in this block $b$.
Note that the positions in the conditioning prefix ($x^{<b}_0$) are deterministically unmasked.
We have
\begin{align}
    \Pr(E \mid t) &= \alpha_t^{k} \cdot (1-\alpha_t)^{L'-k} \\
    &= (1-t)^{k} \cdot t^{L'-k}
\end{align}
Then marginalize over $t \sim U[0,1]$,
\begin{align}
    \Pr(E) &= \mathbb{E}_{t\sim U[0,1]} \Pr(E \mid t) \\
    &= \int_0^1 (1-t)^{k} \cdot t^{L'-k} dt \\
    &= \frac{k!(L'-k)!}{(L'+1)!} \\
    &= \frac{1}{(L'+1) \binom{L'}{k}}
\end{align}
\end{proof}

\subsection{Proof of Proposition~\ref{prop:alignment}}

\label{sec:prop2_proof}
\begin{proof}
Recall the definition of \emph{AR pattern alignment}:
Under MDM, it is
\begin{align*}
\rho_\ell(\mathbf{M}^{\mathrm{MDM}})
&= \frac{1}{L-1}
\Bigg[
\sum_{j=1}^{\ell-1}(1-\mathbf{M}^{\mathrm{MDM}}_j) \\
&\qquad\qquad\qquad
+ \sum_{j=\ell+1}^{L}\mathbf{M}^{\mathrm{MDM}}_j
\Bigg].
\end{align*}
Similarly, under block diffusion, it is
\begin{align*}
\rho_\ell(\mathbf{M}^{\mathrm{BDLM}})
&= \frac{1}{bL'-1}
\Bigg[
\sum_{j=1}^{\ell-1}(1-\mathbf{M}^{\mathrm{BDLM}}_j) \\
&\qquad\qquad\qquad
+ \sum_{j=\ell+1}^{bL'}\mathbf{M}^{\mathrm{BDLM}}_j
\Bigg].
\end{align*}

For the MDM case,
for all positions $j \in \{1, ...,L\}$, 
$\mathbf{M}^{\text{MDM}}_j \mid t \sim \text{Bernoulli}(1-\alpha_t)$ iid, with $t \sim U[0,1]$.
Assume the linear schedule $\alpha_t=1-t$.
\begin{align*}
    \mathbb{E} [\mathbf{M}^{\text{MDM}}_j] &= \mathbb{E}_{t\sim U[0,1]}[\mathbb{E} [\mathbf{M}^{\text{MDM}}_j \mid t]] \\
    &= \mathbb{E}_{t\sim U[0,1]}[1-\alpha_t] \\
    &= \int_0^1 t \, dt =\frac{1}{2}.
\end{align*}
Similarly,
\begin{align*}
    \mathbb{E} [1-\mathbf{M}^{\text{MDM}}_j] &= \mathbb{E}_{t\sim U[0,1]}[\mathbb{E} [1-\mathbf{M}^{\text{MDM}}_j \mid t]] \\
    &= \int_0^1 (1-t) \, dt =\frac{1}{2}.
\end{align*}
Therefore,
\begin{align*}
\mathbb{E}
[
\rho_\ell(\mathbf{M}^{\text{MDM}})
] &= \frac{1}{L-1} \bigg[(\ell-1) \frac{1}{2} + (L-\ell)\frac{1}{2} \bigg] \\
        &= \frac{1}{2}
\end{align*}

For the block diffusion case,
$\mathbf{M}^{\text{BDLM}}_j=0$ for $j \in \{1,...,(b-1)L'\}$ deterministically.
For positions $j \in \{(b-1)L' + 1, ..., bL'\}$,
$\mathbf{M}^{\text{BDLM}}_j \mid t \sim \text{Bernoulli}(1-\alpha_t)$ iid, with $t \sim U[0,1]$.
Let $k=\ell-1-(b-1)L'$ be the number of positions before $\ell$ in this block $b$.
Let $r=(b-1)L'$, we split the first sum at the block-$b$ boundary:
\begin{align*}
\mathbb{E} \bigg[
\sum_{j=1}^{\ell-1} (1-\mathbf{M}^{\mathrm{BDLM}}_j)
\bigg]
&= \sum_{j=1}^{r} 
\mathbb{E}[1-\mathbf{M}^{\mathrm{BDLM}}_j] \\
& + \sum_{j=r+1}^{\ell-1} 
\mathbb{E}[1-\mathbf{M}^{\mathrm{BDLM}}_j] \\
&= (b-1)L' + \frac{k}{2}.
\end{align*}
For the second sum,
\begin{align*}
    \mathbb{E} \bigg[
    \sum_{j=\ell+1}^{bL'} \mathbf{M}^{\text{BDLM}}_j
    \bigg]
    &= \frac{L'-k-1}{2}
\end{align*}
Therefore, combining both expected sums,
\begin{align*}
& \mathbb{E}
[
\rho_\ell(\mathbf{M}^{\text{BDLM}})
] \\
&= \frac{1}{bL'-1} \bigg[
           (b-1) L' + \frac{k}{2} + \frac{L'-k-1}{2}
           \bigg] \\
        &= \frac{1}{2} + \frac{(b-1)L'}{2(bL'-1)}
\end{align*}

\end{proof}

\section{Experimental Details}
\label{sec:appendix_setup}

\subsection{Model Checkpoints}
\label{sec:appendix_checkpoints}

We use the following publicly available Hugging Face checkpoints in our experiments.

\begin{itemize}
    \item \textbf{Fast-dLLM v2 7B.}
    The model is initialized from Qwen2.5-7B-Instruct \citep{qwen2.5} and
    adapted to block diffusion paradigm with $\sim$ 1B tokens of fine-tuning.
    We use \texttt{Efficient-Large-Model/Fast\_dLLM\_v2\_7B}\footnote{\url{https://huggingface.co/Efficient-Large-Model/Fast_dLLM_v2_7B}},
    the 7B checkpoint released by~\citet{fastdllmv2}.

    \item \textbf{SDAR 8B.}
    The model is initialized from Qwen3-8B \citep{qwen3} and
    adapted to the block diffusion paradigm via continued pretraining on $50$B tokens,
    followed by supervised fine-tuning on $4$B instruction tokens.
    We use \texttt{JetLM/SDAR 8B-Chat-b32}\footnote{\url{https://huggingface.co/JetLM/SDAR-8B-Chat-b32}},
    the block-size-$32$ version 8B checkpoint released by~\citet{sdar}.

    \item \textbf{LLaDA 8B.}
    This model is trained from scratch with MDM objective.
    We use \texttt{GSAI-ML/LLaDA 8B-Instruct}\footnote{\url{https://huggingface.co/GSAI-ML/LLaDA-8B-Instruct}},
    released by \citet{llada}.
    This model is used in the comparison experiment in Section~\ref{sec:obs}
    to contrast the decoding behavior of an MDM with that of a BDLM.
    We use it with \texttt{DualCache} from \citet{fast-dllm} for faster inference.

    \item \textbf{LLaDA2.1-Mini 16B.}
    A Mixture-of-Experts (MoE) block diffusion language model.
    Generation follows a block-wise denoising loop with
    token-to-token (T2T) editing within each block.
    We use \texttt{inclusionAI/LLaDA2.1-mini}\footnote{\url{https://huggingface.co/inclusionAI/LLaDA2.1-mini}},
    released by~\citet{llada21}.
\end{itemize}

All checkpoints are loaded and evaluated in \texttt{bfloat16}.

\subsection{Hyperparameter Values and APD Verifier Models}
\label{sec:appendix_thresholds}
In the speed--quality trade-off experiments, we use the same set of threshold values for both Fast-dLLM v2 and SDAR.
Specifically, we use $\tau_c \in \{0.7, 0.75, 0.8, 0.85, 0.9\}$ for {\small\sf Confidence},
$\tau_m \in \{0.7, 0.8, 0.9, 0.95, 0.97\}$ for {\small\sf Margin}, and
$\tau_h \in \{0.1, 0.2, 0.3, 0.4, 0.5, 0.6, 0.7\}$ for {\small\sf Entropy}.

For the comparison with advanced parallel samplers,
we use $\tau_c \in \{0.7, 0.75, 0.8, 0.85, 0.9\}$ for confidence-based PARD,
entropy-bound thresholds in $\{0.01, 0.05, 0.1, 0.2, 0.3, 0.4\}$ for EB-Sampler,
KL-divergence thresholds in $\{0.01, 0.1, 1.0, 10.0\}$ for KLASS,
and mixture weights in $\{0.01, 0.1, 0.3, 0.5\}$ for APD.
For Hierarchy, which uses high- and low-confidence thresholds as well as a remasking step, 
we sweep the high threshold $\tau_{\mathrm{high}}$ over the same grid as $\tau_c$,
since both thresholds play a similar role in controlling confidence-based token acceptance.
We set the low threshold to $\tau_{\mathrm{low}}=0.5$, as suggested by the authors~\cite{hierarchy}, and disable remasking, as we empirically found it to degrade performance.

APD requires an external small AR verifier that shares the same tokenizer as the target DLM. 
Following the original APD setup~\cite{apd}, we use verifiers at the 0.5--0.6B scale: Qwen2.5-0.5B for Fast-dLLM v2 and Qwen3-0.6B for SDAR, matching the model family of each BDLM. 
We omit APD for LLaDA2.1-Mini because we could not find a small AR verifier that shares its tokenizer.

\subsection{Implementation Details}
For all benchmarks and all samplers, we use greedy decoding (temperature $=0$, top-$p$ disabled).
We apply each model's default chat template to all question prompts.
Generation stops at the end-of-sequence token or when the per-task \texttt{max\_new\_tokens} budget is reached.
All experiments run on a single NVIDIA A6000 GPU.

\begin{table}[ht]
\centering
\small
\setlength{\tabcolsep}{4pt}
\renewcommand{\arraystretch}{1.1}
\begin{tabular}{lcc}
\toprule
\textbf{Benchmark} & \textbf{\#-shot} & \textbf{Max new tokens} \\
\midrule
HumanEval & 0 & 768  \\
MBPP      & 0 & 768  \\
GSM8K     & 0 & 2048 \\
MATH      & 0 & 2048 \\
IFEval    & 0 & 2048 \\
BBH       & 0 & 1024 \\
\bottomrule
\end{tabular}
\caption{Per-benchmark evaluation setup.
    ``\#-shot'' counts in-context demonstrations, we use zero-shot for all benchmarks in our evaluation.
}
\label{tab:eval_setup}
\vspace{-5pt}
\end{table}

\subsection{Benchmarks}
\label{sec:appendix_benchmarks}
We use \texttt{evalplus}~\citep{evalplus} for HumanEval and MBPP,
and the \texttt{lm-evaluation-harness}~\citep{lmevalharness} for the remaining tasks.
Table~\ref{tab:eval_setup} summarizes the per-benchmark configuration,
including the number of in-context demonstrations and the maximum new tokens allowed for generation.

For HumanEval and MBPP we report both the Base and extended Plus pass@1,
where a problem is counted as solved under Plus only if it passes both the base and extended test suites.
For IFEval we report the prompt-level strict accuracy.

For LLaDA2.1 BBH, due to computational constraints, we evaluate on a balanced
$8$-task subset of BBH, with
$4$ tasks drawn from each of the two task groupings introduced by
\citet{bbh}:
\begin{itemize}
    \item \textbf{Algorithmic and multi-step arithmetic reasoning}:
        \texttt{boolean\_expressions}, \texttt{navigate},
        \texttt{object\_counting}, \texttt{web\_of\_lies}.
    \item \textbf{Natural language understanding}:
        \texttt{causal\_judgement}, \texttt{date\_understanding},
        \texttt{logical\_deduction\_seven\_objects},
        \texttt{temporal\_sequences}.
\end{itemize}

We report the macro-mean over the eight tasks, following the
convention of \citet{bbh}.

\section{Additional Results}

\subsection{AR Bias Analysis on Dream}
\label{sec:appendix_dream_ar}
Dream is an MDM trained from an AR initialization, unlike LLaDA, which is trained from scratch.
To disentangle the effects of AR initialization and block diffusion training on AR bias, we compare Dream with SDAR in this section.
We use \texttt{Dream-org/Dream-v0-Instruct-7B}.

\begin{table}[H]
\centering
\resizebox{\linewidth}{!}{%
\begin{tabular}{lccc}
\toprule
Method & SDAR 8B & LLaDA 8B & Dream 7B \\
\midrule
Confidence & 75.0 & 45.7 & 51.8 \\
Entropy & 77.4 & 45.7 & 50.6 \\
Margin & 71.3 & 43.3 & 45.7 \\
AR              & 79.3 & 41.5 & 52.4 \\
\bottomrule
\end{tabular}
}
\caption{HumanEval base pass@1 (\%) under arbitrary-order sampling and pure AR-order sampling.}
\label{tab:humaneval-dream-quality-gap}
\end{table}

Table \ref{tab:humaneval-dream-quality-gap} shows the pass@1 values obtained from pure AR sampling with the samplers 
({\small\sf confidence/entropy/margin})
across SDAR, LLaDA, and Dream on \textit{HumanEval}.
On Dream, pure AR sampling achieves the highest pass@1 (52.4).
Note that this is consistent with Figure 3 of \citet{flexibilitytrap},
which likewise reports higher \textit{HumanEval} pass@1 for AR decoding than for arbitrary-order decoding on Dream. 
Nevertheless, the gap between AR and the best sampler is only 0.6 points for Dream, 
compared with 1.9 points for SDAR. 
Thus, its AR bias remains weaker than that of SDAR. 
To further validate this observation, we compare the local and global AR-ness at different $k$ for Dream, SDAR, and LLaDA in 
Figure \ref{fig:dream_ar_ness}.

\begin{figure}[H]
    \centering
    \includegraphics[width=\linewidth]{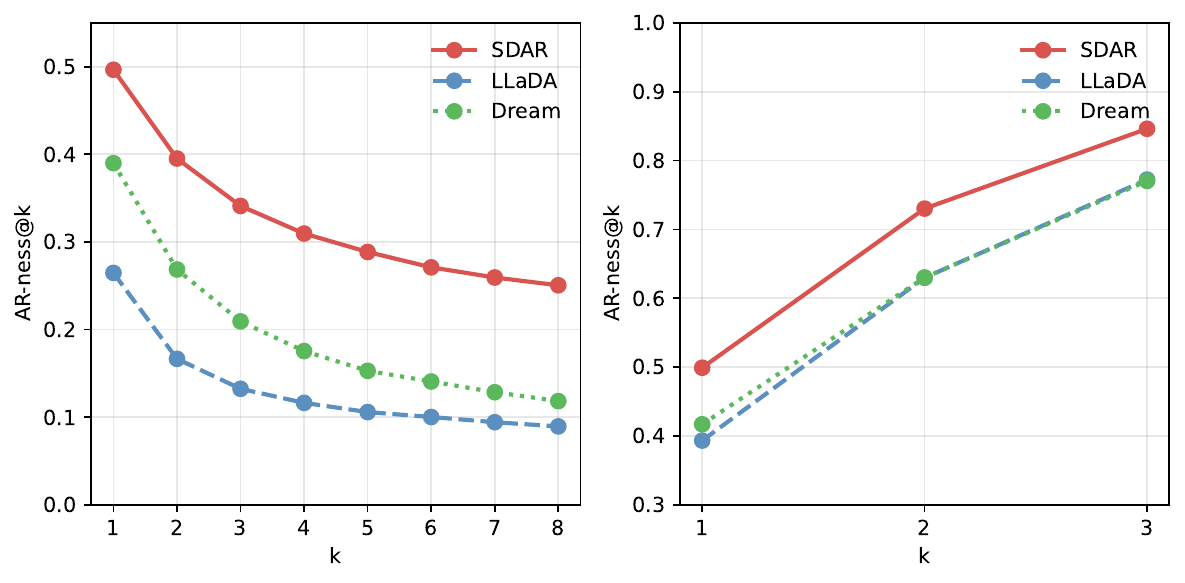}
    \caption{
    Local (left) and global (right) AR-ness@$k$ of decoding on {\it HumanEval} under Confidence sampling.
    }
    \label{fig:dream_ar_ness}
    \vspace{-5pt}
\end{figure}

Dream generally exhibits higher AR-ness than LLaDA. 
However, despite also being initialized from an AR model, its AR-ness remains consistently lower than that of SDAR across all values of $k$.
This provides additional evidence that block-diffusion training contributes to AR bias beyond AR initialization.

\subsection{Speed-Quality Trade-off}
\label{sec:appendix_tradeoff}
Figure \ref{fig:humaneval-nfe-2x3} compares
{\small\sf Confidence}, {\small\sf Margin}, and {\small\sf Entropy} with
their PARD variants, and 
Figure \ref{fig:humaneval-nfe-1x2} compares confidence-based PARD with
EB-Sampler, KLASS, APD, and Hierarchy.
Both of the figures show the pass@1-NFE trade-offs on \textit{HumanEval}.
PARD consistently achieves strong trade-offs in terms of NFE on \textit{HumanEval}.

\begin{figure*}[t]
    \centering
    \includegraphics[width=\linewidth]{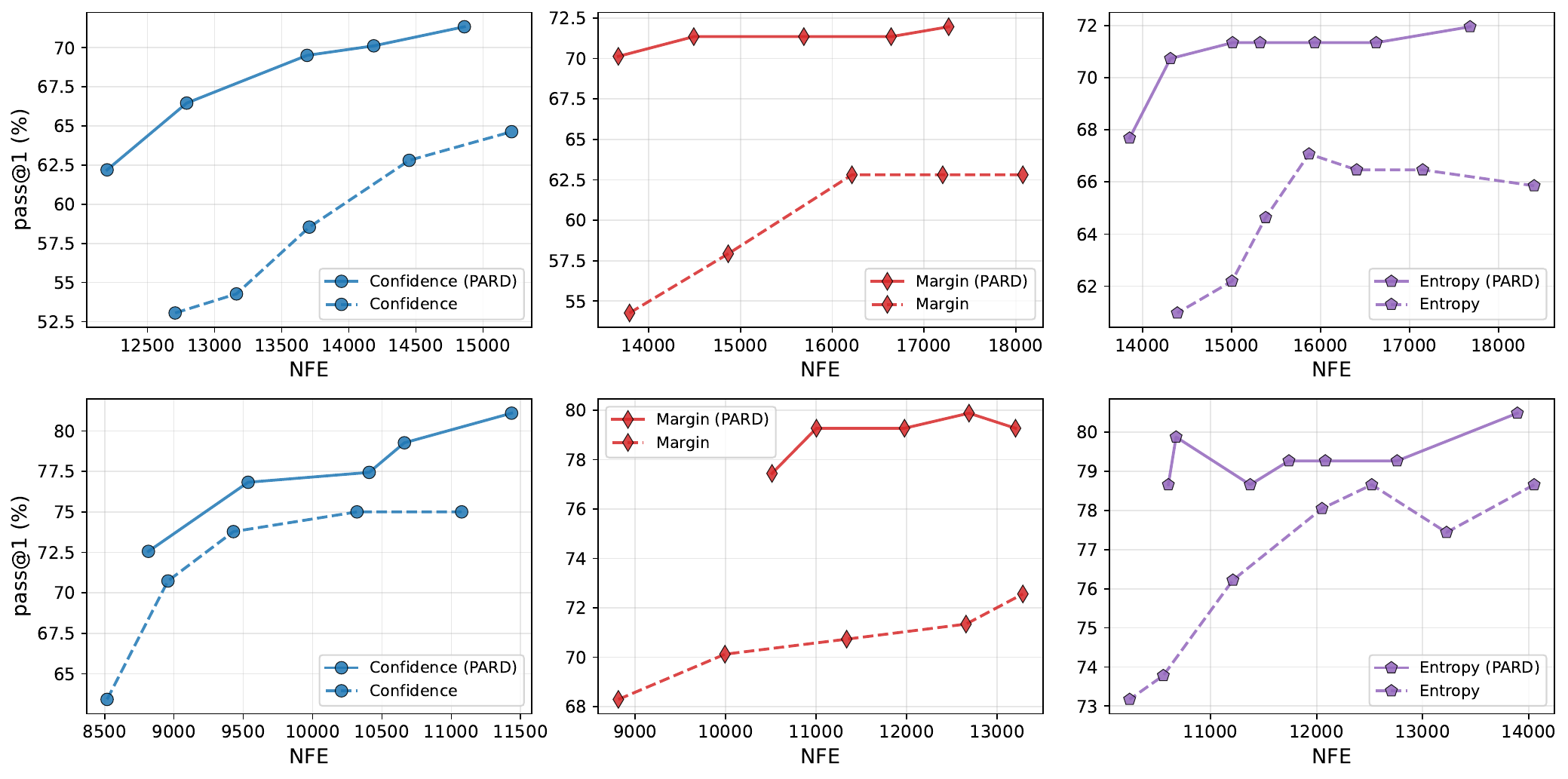}
    \caption{
        \textit{HumanEval} pass@1 vs.\ Number of Function Evaluations (NFE)
        for {\small\sf Confidence} (left), {\small\sf Margin} (middle), and {\small\sf Entropy} (right) parallel sampling and their
PARD variants on Fast-dLLM v2 7B (top row) and SDAR 8B (bottom row).
    }
    \label{fig:humaneval-nfe-2x3}
\end{figure*}

\begin{figure*}[t]
    \centering
    \includegraphics[width=\linewidth]{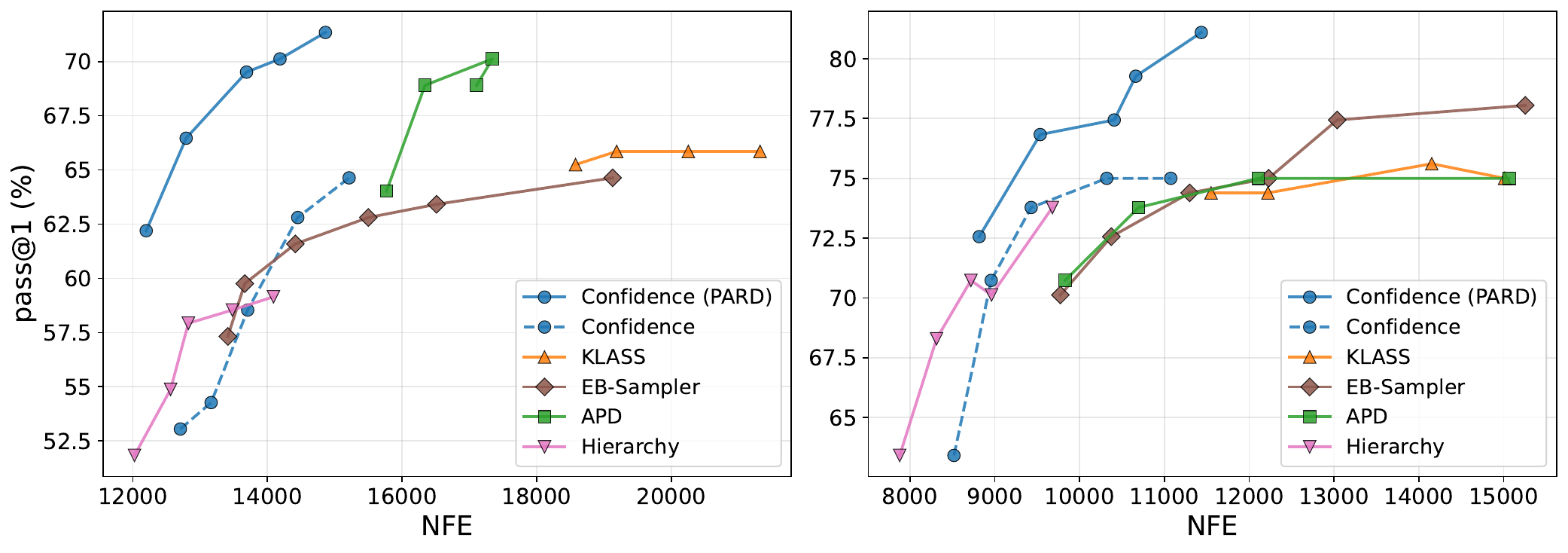}
    \caption{
        {\it HumanEval} Base pass@1 vs.\ Number of Function Evaluations (NFE)
        for confidence-based PARD, EB-Sampler, KLASS, and APD on Fast-dLLM v2 7B (left) and SDAR 8B (right).
    }
    \label{fig:humaneval-nfe-1x2}
\end{figure*}

Figures~\ref{fig:mbpp-tps-2x3} and~\ref{fig:mbpp-nfe-2x3} show the
corresponding \textit{MBPP} pass@1--throughput and pass@1--NFE trade-offs for
{\small\sf Confidence}, {\small\sf Margin}, and {\small\sf Entropy}, 
along with their PARD variants. 
Figures~\ref{fig:mbpp-tps-1x2}
and~\ref{fig:mbpp-nfe-1x2} show the same MBPP trade-offs for confidence-based PARD, EB-Sampler, KLASS, APD, and Hierarchy.
PARD achieves stronger trade-offs on \textit{MBPP} than the other methods, except for SDAR 8B, where EB-Sampler performs best.

\begin{figure*}[t]
    \centering
    \includegraphics[width=\linewidth]{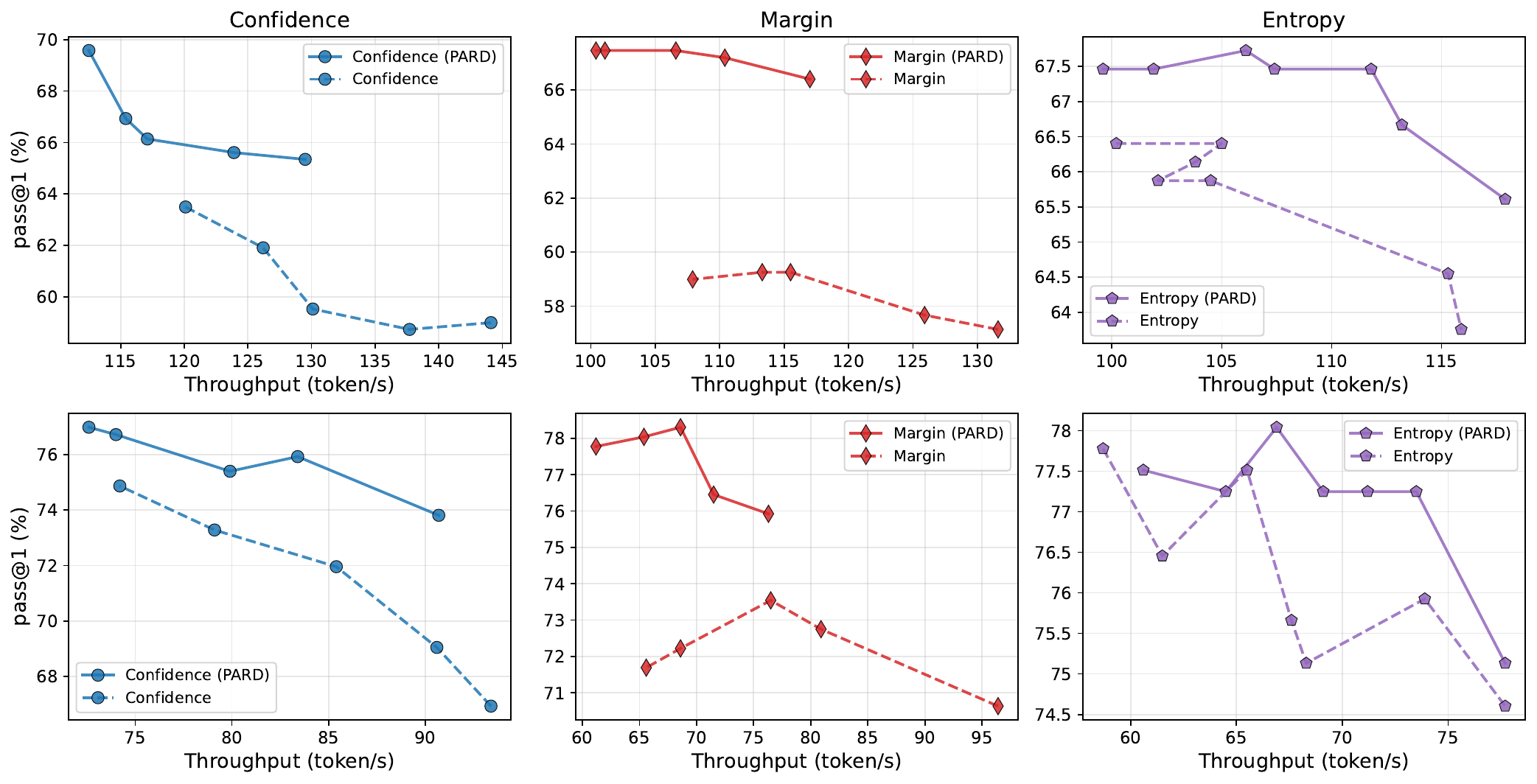}
    \caption{
        \textit{MBPP} pass@1 vs.\ decoding throughput
        for {\small\sf Confidence} (left), {\small\sf Margin} (middle), and {\small\sf Entropy} (right) parallel sampling and their
PARD variants on Fast-dLLM v2 7B (top row) and SDAR 8B (bottom row).
    }
    \label{fig:mbpp-tps-2x3}
\end{figure*}

\begin{figure*}[t]
    \centering
    \includegraphics[width=\linewidth]{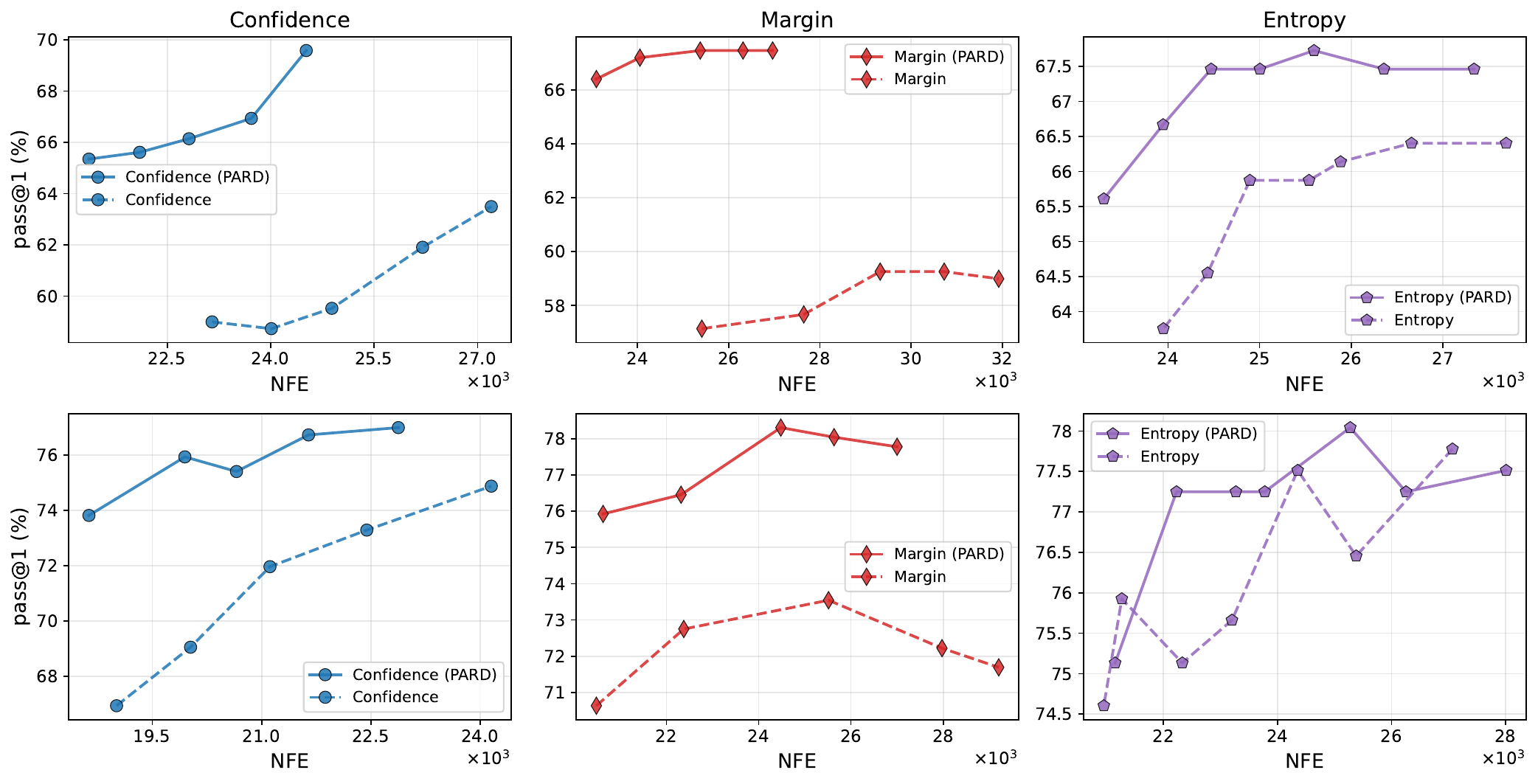}
    \caption{
        \textit{MBPP} pass@1 vs.\ NFE
        for {\small\sf Confidence} (left), {\small\sf Margin} (middle), and {\small\sf Entropy} (right) parallel sampling and their
PARD variants on Fast-dLLM v2 7B (top row) and SDAR 8B (bottom row).
    }
    \label{fig:mbpp-nfe-2x3}
\end{figure*}

\begin{figure*}[t]
    \centering
    \includegraphics[width=\linewidth]{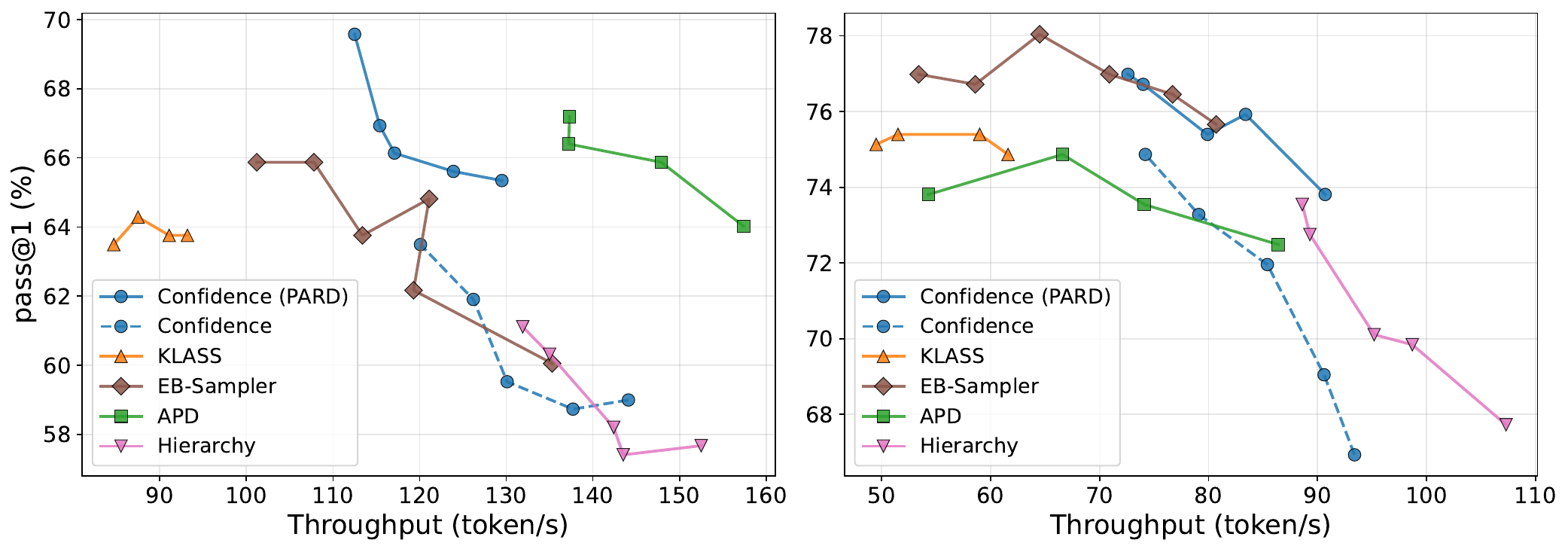}
    \caption{
        {\it MBPP} Base pass@1 vs.\ decoding throughput
        for confidence-based PARD, EB-Sampler, and KLASS on Fast-dLLM v2 7B (left) and SDAR 8B (right).
    }
    \label{fig:mbpp-tps-1x2}
\end{figure*}

\begin{figure*}[t]
    \centering
    \includegraphics[width=0.9\linewidth]{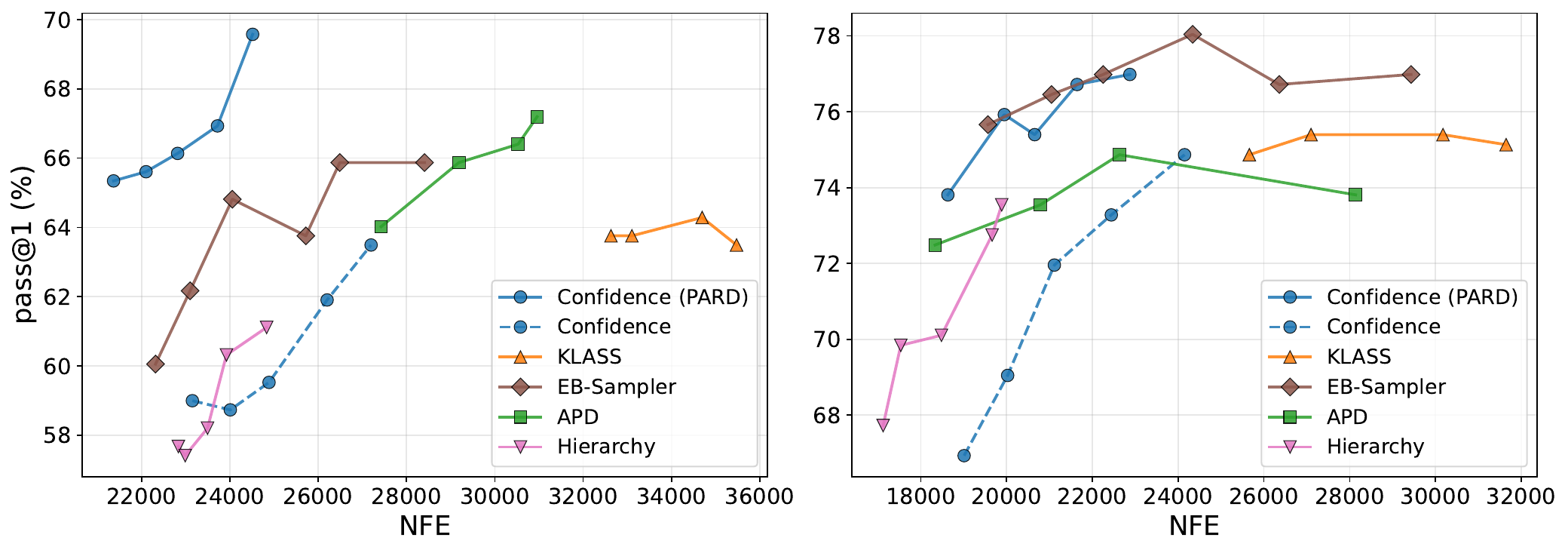}
    \caption{
        {\it MBPP} Base pass@1 vs.\ NFE
        for confidence-based PARD, EB-Sampler, and KLASS on Fast-dLLM v2 7B (left) and SDAR 8B (right).
    }
    \label{fig:mbpp-nfe-1x2}
\end{figure*}

\clearpage

\twocolumn[{%
\begin{@twocolumnfalse}
\centering
\includegraphics[width=0.95\textwidth]{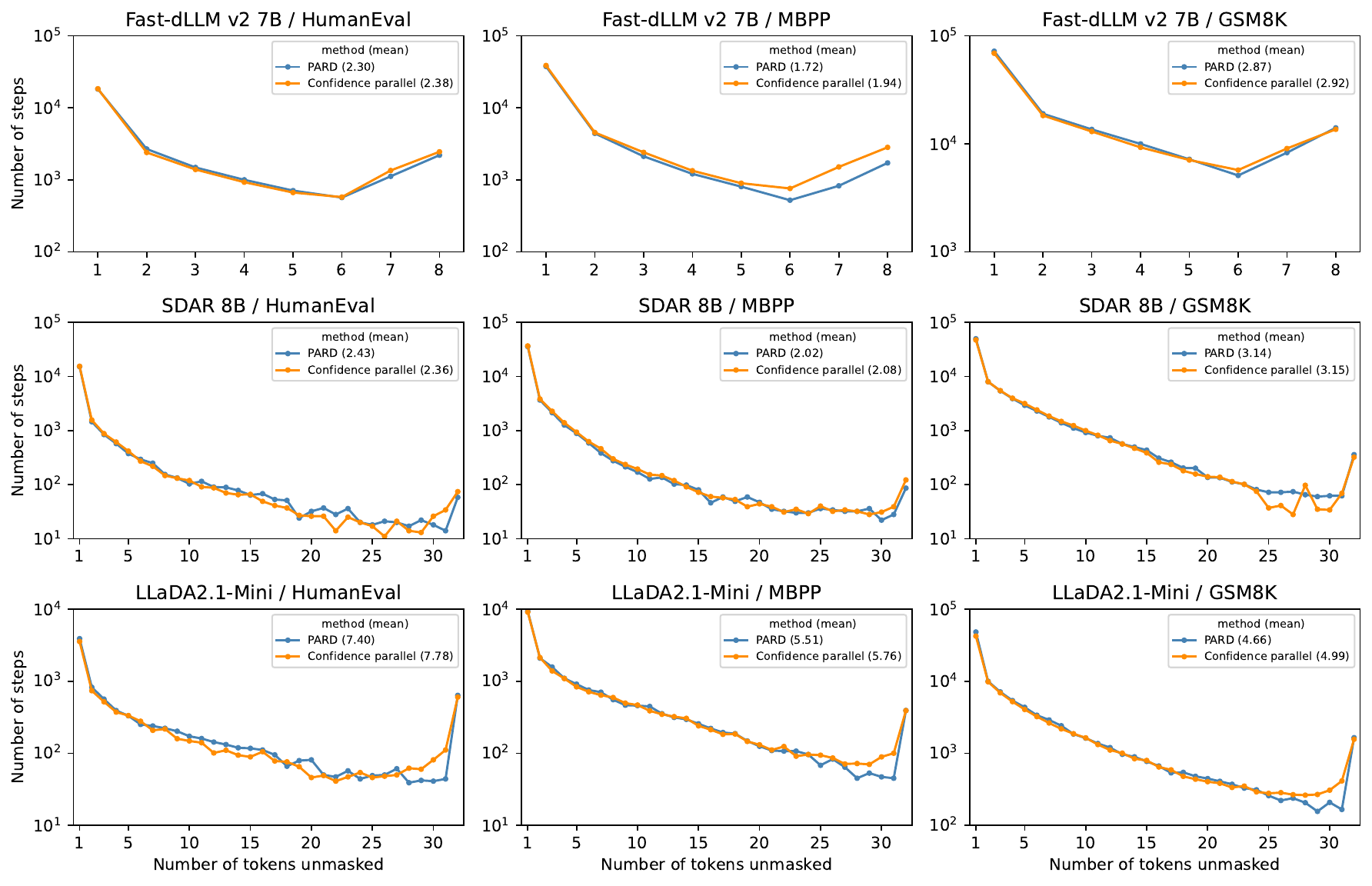}
\captionof{figure}{Number of tokens unmasked per step by confidence-based PARD and Confidence Parallel on \textit{HumanEval}, \textit{MBPP}, and \textit{GSM8K}.}
\label{fig:step_unmask_merged}
\vspace{1.0em}

\end{@twocolumnfalse}
}]

\subsection{Number of Tokens Unmasked per Step}
\label{sec:appendix_num_unmasked}
Figure~\ref{fig:step_unmask_merged} shows the distributions of the number of tokens unmasked per step by confidence-based PARD and Confidence Parallel on \textit{HumanEval}, \textit{MBPP}, and \textit{GSM8K}.
Under the default threshold configurations as in Section \ref{sec:benchmark_performance}, 
Fast-dLLM v2 and SDAR use $\tau_c=0.9$, while LLaDA2.1 uses a more permissive threshold of $\tau_c=0.7$.
We observe similar distributions for PARD and Confidence Parallel, suggesting that PARD preserves most of the parallelism while enforcing a left-to-right prefix constraint.
For Fast-dLLM v2, the number of tokens unmasked in a single step is upper-bounded by 8 because we follow its default sub-block inference strategy with sub-block size 8.
Accordingly, Fast-dLLM v2 and SDAR generally unmask about 2-3 tokens per step on average, 
whereas LLaDA2.1 shows substantially higher parallelism, with average unmasking counts around 5-8 tokens per step.

\subsection{Qualitative Case Study}
\label{sec:appendix_case_study}

We study how arbitrary-order decoding can fail with a concrete example from \textit{HumanEval}.
Figure~\ref{fig:example_histogram} shows the generations produced by Confidence Parallel and confidence-based PARD on \texttt{HumanEval/111}, using SDAR 8B and the same default confidence threshold $\tau_c$.
The two outputs are nearly identical and differ only in the empty-input guard \texttt{if not test:\ return \{\}}.
PARD generates this guard, whereas Confidence Parallel omits this part.
Without the guard, \texttt{max(counts.values())} is evaluated on an empty dictionary when the input is the empty string and raises a \texttt{ValueError}, 
so Confidence Parallel fails the \texttt{histogram("")} test while PARD passes.

The omission stems from the decoding order, shown in Figure~\ref{fig:example_steps}.
Both samplers produce the function signature identically through step 0 and step 1.
At step 2, the guard token \texttt{"if"}, at the slot immediately after the signature, has confidence $0.20$, 
while \texttt{"="}, one slot later (starting \texttt{letters = test.split()}), has confidence $0.28$.
Confidence Parallel selects the most confident tokens first, so it unmasks \texttt{"="}.
This fixes the surrounding body and forces the earlier slot to become \texttt{letters}, leaving no room for the guard.
PARD instead commits positions strictly left to right, so it resolves the earlier slot first and unmasks \texttt{"if"} despite its lower confidence, keeping the guard.

\begin{figure*}[t]
\begin{exbox}{Sampling Example --- HumanEval/111 (\texttt{histogram})}

\textbf{Prompt:}
\begin{Verbatim}[fontsize=\footnotesize, breaklines, breakanywhere]
def histogram(test):
    """Given a string representing a space separated lowercase letters, return a
    dictionary of the letter with the most repetition and containing the corresponding
    count. If several letters have the same occurrence, return all of them.

    Example:
    histogram('a b c') == {'a': 1, 'b': 1, 'c': 1}
    histogram('a b b a') == {'a': 2, 'b': 2}
    histogram('b b b b a') == {'b': 4}
    histogram('') == {}
    """
\end{Verbatim}

\smallskip
\textbf{Confidence Parallel} \textcolor{failc}{\textbf{(Fail)}}\textbf{:}
{\footnotesize\begin{alltt}
def histogram(test):
    letters = test.split()
    counts = \{\}
    for letter in letters:
        if letter in counts:
            counts[letter] += 1
        else:
            counts[letter] = 1
    max_count = max(counts.values())
    return \{letter: count for letter, count in counts.items() if count == max_count\}
\end{alltt}}

\smallskip
\textbf{PARD} \textcolor{passc}{\textbf{(Pass)}}\textbf{:}
{\footnotesize\begin{alltt}
def histogram(test):
    \gline{if not test:}
        \gline{return \{\}}
    letters = test.split()
    counts = \{\}
    for letter in letters:
        if letter in counts:
            counts[letter] += 1
        else:
            counts[letter] = 1
    max_count = max(counts.values())
    return \{letter: count for letter, count in counts.items() if count == max_count\}
\end{alltt}}

\end{exbox}
\caption{Generations of Confidence Parallel and PARD on \texttt{HumanEval/111}. 
The empty-input guard is highlighted in green.}
\label{fig:example_histogram}
\end{figure*}

\begin{figure*}[h]
\centering
\begin{tabular}{@{}c l l@{}}
\textbf{Step} & \textbf{Confidence Parallel} & \textbf{PARD}\\
\midrule
0 & \texttt{def} & \texttt{def}\\
1 & \texttt{histogram(test):} & \texttt{histogram(test):}\\
2 & \textcolor{failc}{\texttt{=}}\,{\footnotesize(0.28)} & \textcolor{guardtext}{\texttt{if}}\,{\footnotesize(0.20)}\\
$\vdots$ & \texttt{letters = test.split()}~\textcolor{failc}{\ding{55}}
  & \texttt{if not test: return \{\}}~\textcolor{passc}{\ding{51}}\\
\end{tabular}
\caption{Tokens unmasked at the first few decoding step by Confidence Parallel and PARD for the example in Figure~\ref{fig:example_histogram}. Values in parentheses are the predictive confidence in the token.}
\label{fig:example_steps}
\end{figure*}

\subsection{Top-$p$ Sampling Experiments}
\label{sec:appendix_top-p}
All experiments in Section \ref{sec:experiments} use greedy decoding.
To assess whether PARD's advantage remains robust under stochastic nucleus decoding,
we additionally evaluate top-$p$ sampling on \textit{HumanEval}, \textit{MBPP}, and \textit{GSM8K} using temperature = $1.0$ and $p$ = 0.8, 
a configuration adopted in recent LLM reasoning studies \citep{liu2025your,jinentropy}.
We compare PARD against the three most competitive parallel baselines used in our main experiments:
Confidence Parallel, KLASS, and EB-Sampler.
We repeat the experiment three times and report the mean and standard deviation of pass@1.

\begin{table*}[t]
\centering
\resizebox{\textwidth}{!}{%
\begin{tabular}{llccccc}
\toprule
Model & Method & \multicolumn{2}{c}{HumanEval} & \multicolumn{2}{c}{MBPP}
& GSM8K \\
& & Base & Plus & Base & Plus & \\
\midrule
Fast-dLLM v2 7B
& KLASS
& 64.0{\ensuremath{\pm}}2.6 & 58.9{\ensuremath{\pm}}2.0
& 56.9{\ensuremath{\pm}}0.6 & 48.1{\ensuremath{\pm}}1.3
& 82.0{\ensuremath{\pm}}0.6 \\
& EB-Sampler
& 61.4{\ensuremath{\pm}}2.5 & 55.5{\ensuremath{\pm}}2.6
& 58.8{\ensuremath{\pm}}0.4 & 49.8{\ensuremath{\pm}}1.0
& 81.3{\ensuremath{\pm}}0.3 \\
& Confidence Parallel
& 58.1{\ensuremath{\pm}}3.7 & 53.9{\ensuremath{\pm}}3.7
& 57.5{\ensuremath{\pm}}1.5 & 49.1{\ensuremath{\pm}}2.2
& 81.9{\ensuremath{\pm}}0.4 \\
& PARD
& \textbf{67.1{\ensuremath{\pm}}1.8} & \textbf{62.6{\ensuremath{\pm}}1.2}
& \textbf{63.4{\ensuremath{\pm}}1.6} & \textbf{53.5{\ensuremath{\pm}}1.5}
& \textbf{82.2{\ensuremath{\pm}}0.9} \\
\midrule
SDAR 8B
& KLASS
& 67.3{\ensuremath{\pm}}2.7 & 63.4{\ensuremath{\pm}}2.2
& 69.8{\ensuremath{\pm}}1.3 & 59.1{\ensuremath{\pm}}1.2
& 90.3{\ensuremath{\pm}}0.6 \\
& EB-Sampler
& 71.5{\ensuremath{\pm}}2.4 & 64.8{\ensuremath{\pm}}1.0
& 72.7{\ensuremath{\pm}}0.8 & 61.1{\ensuremath{\pm}}1.6
& 90.0{\ensuremath{\pm}}0.4 \\
& Confidence Parallel
& 70.9{\ensuremath{\pm}}1.7 & 66.9{\ensuremath{\pm}}2.5
& 72.6{\ensuremath{\pm}}1.9 & 61.3{\ensuremath{\pm}}1.3
& 89.8{\ensuremath{\pm}}0.3 \\
& PARD
& \textbf{73.8{\ensuremath{\pm}}2.8} & \textbf{67.9{\ensuremath{\pm}}2.4}
& \textbf{72.9{\ensuremath{\pm}}0.7} & \textbf{61.8{\ensuremath{\pm}}1.1}
& \textbf{90.8{\ensuremath{\pm}}0.5} \\
\midrule
LLaDA2.1-Mini
& KLASS
& 73.6{\ensuremath{\pm}}2.9 & 70.3{\ensuremath{\pm}}2.7
& 78.4{\ensuremath{\pm}}0.5 & 65.5{\ensuremath{\pm}}1.5
& 90.2{\ensuremath{\pm}}0.3 \\
& EB-Sampler
& 76.6{\ensuremath{\pm}}1.3 & 74.0{\ensuremath{\pm}}1.0
& 79.6{\ensuremath{\pm}}0.2 & 67.7{\ensuremath{\pm}}0.6
& 90.8{\ensuremath{\pm}}0.4 \\
& Confidence Parallel
& 79.9{\ensuremath{\pm}}2.5 & 76.8{\ensuremath{\pm}}2.3
& 78.2{\ensuremath{\pm}}0.7 & 66.5{\ensuremath{\pm}}0.2
& 90.6{\ensuremath{\pm}}0.4 \\
& PARD
& \textbf{85.8{\ensuremath{\pm}}1.0} & \textbf{81.9{\ensuremath{\pm}}0.3}
& \textbf{82.6{\ensuremath{\pm}}0.3} & \textbf{70.6{\ensuremath{\pm}}0.2}
& \textbf{91.1{\ensuremath{\pm}}0.3} \\
\bottomrule
\end{tabular}
}
\caption{
Performance of Fast-dLLM v2 7B, SDAR 8B, and LLaDA2.1-Mini under top-$p$ sampling with temperature $1.0$ and $p=0.8$ on 
\textit{HumanEval}, \textit{MBPP}, and \textit{GSM8K}. 
Results are reported as mean $\pm$ standard deviation over three runs.
The best result for each model and benchmark is bolded.
}
\label{tab:top-p}
\end{table*}

Table \ref{tab:top-p} shows the results on Fast-dLLM v2, SDAR, and LLaDA2.1-Mini.
We have two major observations.
First, 
under top-$p$ sampling, almost all samplers perform worse than their greedy counterparts in Table \ref{tab:main_results}.
This is consistent with the default greedy decoding adopted by recent BDLMs such as Fast-dLLM v2 and LLaDA2.1, 
while also avoiding an additional hyperparameter search.
Second, 
PARD achieves the highest mean performance among all parallel samplers, showing that its advantage remains robust under stochastic top-$p$ sampling.
We conduct paired t-tests over all models and benchmarks for PARD vs each baseline. 
PARD significantly outperforms Confidence Parallel, EB-Sampler, and KLASS, with $p$-values of $3.5\times10^{-4}$, $3.4\times10^{-4}$, and $1.7\times10^{-4}$, respectively.
Overall, PARD retains its relative advantage under top-$p$ sampling.

\clearpage

\end{document}